\documentclass{article}

\PassOptionsToPackage{square, numbers}{natbib}

\usepackage[final]{neurips_2021}
\usepackage[utf8]{inputenc} % allow utf-8 input
\usepackage[T1]{fontenc}    % use 8-bit T1 fonts
\usepackage{hyperref}       % hyperlinks
\usepackage{url}            % simple URL typesetting
\usepackage{booktabs}       % professional-quality tables
\usepackage{amsfonts}       % blackboard math symbols
\usepackage{nicefrac}       % compact symbols for 1/2, etc.
\usepackage{microtype}      % microtypography
\usepackage{xcolor}         % colors

\usepackage{amsmath}
\usepackage{algorithm}
\usepackage{algorithmic}
\usepackage{multirow}          
\usepackage{caption}          
\usepackage{subcaption}          
\usepackage{float}  
\usepackage{amsthm}

\newcommand{\set}[1]{\mathcal{#1}}

\renewcommand{\vec}[1]{\boldsymbol{#1}}
 
 \newcommand{\tblue}[1]{{#1}}

\DeclareMathOperator*{\argmin}{arg\,min}

\usepackage{graphicx}
\graphicspath{ {./figures/} }

\newcommand{\sota}{state-of-the-art~}

\newtheorem{property}{Property}
\newtheorem{lemma}{Lemma}

\renewcommand{\algorithmiccomment}[1]{\bgroup\hfill//~#1\egroup}

\makeatletter
\def\blfootnote{\gdef\@thefnmark{}\@footnotetext}
\makeatother

 \title{Gradient Inversion with Generative Image Prior}

\author{%
  \textbf{Jinwoo Jeon$^{1*}$, Jaechang Kim$^{2*}$, Kangwook Lee$^{3}$, Sewoong Oh$^{4}$, Jungseul Ok$^{1,2}$}
  \\
  $^{1}$ Department of Computer Science \& Engineering, Pohang University of Science and Technology\\
  $^{2}$ Graduate School of Artificial Intelligence, Pohang University of Science and Technology\\
  $^{3}$ Department of Electrical and Computer Engineering, University of Wisconsin-Madison, Madison\\
  $^{4}$ Paul G. Allen School of Computer Science \& Engineering, University of Washington \\
}

\begin{document}

\maketitle

\blfootnote{$^*$ equal contribution}

\begin{abstract}

Federated Learning (FL) is a distributed learning framework, in which 
the local data never leaves clients' devices to preserve privacy, and the server trains models on the data via 
accessing only the gradients of those local data. 
Without further privacy mechanisms such as differential privacy, this leaves the system vulnerable against an attacker who inverts those gradients to reveal clients' sensitive data.
However, a gradient is often insufficient to reconstruct the user data without any prior knowledge.
By exploiting a generative model
pretrained on the data distribution, we demonstrate that data privacy can be easily breached.
% \tblue{
Further, when such prior knowledge 
is unavailable, 
we investigate the possibility of learning the prior from a sequence of gradients seen in the process of FL training.
% }
We experimentally show that
the prior in a form of generative model
is learnable from iterative interactions in FL.
%and brings a significant threat comparable to somewhat explicit prior.
% \tblue{
Our findings strongly suggest that additional mechanisms are necessary
to prevent privacy leakage in FL.
% }

\end{abstract}

\section{Introduction}
\label{submission}

Federated learning (FL) is an emerging framework for distributed learning, where
central server aggregates model updates, rather than user data, from end users~\cite{brisimi2018federated,fedavg}.
The main premise of federated learning is that this particular way of distributed learning can protect users' data privacy as there is no explicit data shared by the end users with the central server. 

\tblue{
However, a recent line of  work~\cite{zhu,zhao,geiping,nvidia} demonstrates that one may recover the private user data used for training by observing the gradients.
This process of recovering the training data from gradients, so-called \emph{gradient inversion}, 
}
poses a huge threat to the federated learning community, as it may imply the fundamental flaw of its main premise.

Even more worryingly, recent works suggest that such gradient inversion attacks can be made even stronger if  certain side-information is available.
For instance, \citet{geiping} show that if the attacker knows a prior that user data consists of natural images, then the gradient inversion attack can leverage such prior, achieving a more accurate recovery of the user data.
Another instance is when batch norm statistics are available at the attacker in addition to gradients. 
This can actually happen if the end users share their local batch norm statistics as in~\cite{fedavg}. 
\citet{nvidia} show that such batch normalization statistics can significantly improve the strength of the gradient inversion attack, enabling precise recovery of high-resolution images. 

In this paper, we systematically study how one can maximally utilize and even obtain the prior information when inverting gradients.
We first consider the case that the attacker 
has a generative model pretrained on the exact or approximate distribution of the user data as a prior.
For this, we propose an efficient gradient inversion algorithm
that utilizes the generative model prior.
In particular, the algorithm consists of two steps, in which
the first step searches the latent space (of lower dimension) defined by the generative model instead of the ambient input space (of higher dimension), and then the second step adapts the generative model to each input given the gradient. Each step provides substantial improvement in the reconstruction.
We name the algorithm as gradient inversion in alternative spaces (GIAS).
Figure~\ref{fig:authors} represents reconstruction results with the proposed method and existing one.
%where showing the superiority of the proposed GIAS

We then consider a realistic scenario in which 
the user data distribution is not known in advance, and thus the attacker needs to learn it from gradients.
For this scenario, we develop a meta-learning framework, 
called gradient inversion to meta-learn (GIML),
which learns a generative model on user data
from observing and inverting 
multiple gradients computed on the data, e.g. across different FL epochs or participating nodes.
Our experimental results demonstrate that one can learn a generative model via GIML and reconstruct data by making use of the learned generative model.

This implies a great threat on privacy leakage
in FL since our methods can be applied
for any data type in most FL scenarios unless
a specialized architecture prevents the gradient leakage explicitly, e.g., \cite{adfss}.

Our main contributions are as follows: 
\begin{itemize}
\item We introduce GIAS that fully utilizes a pretrained generative model
to invert gradient. In addition, we propose GIML which can train generative model from gradients only in FL.

\item We demonstrate significant privacy leakage occurring 
by GIAS with a pretrained generative model in various FL scenarios
which are challenging to other existing methods, e.g., \cite{geiping, nvidia}.

% where our method provides substantial gain to other existing methods \cite{geiping, nvidia} additionally.

\item  We experimentally show that GIML can learn a generative model
on the user data from only gradients,
which provides the same level of data recovery with a given pretrained model.
To our best knowledge, GIML is the first capable of learning explicit 
prior on a set of gradient inversion tasks.

\item 
We note that a gradient inversion technique defines
a standard on defence mechanism % (e.g., how large additive noise or batch size should be) 
in FL for privacy \cite{wei}.
By substantiating that our proposed methods
are able to break down defense mechanisms 
that were safe according to the previous standard,
%stronger than existing ones,
we give a strong warning to the FL community 
to use a higher standard defined by our attack methods, 
and raise the necessity of a more conservative choice of defense mechanisms.
% emphasizing that gradient inversion based attacks 
% are able to break down defense methods 
% which were safe according to the previous standard.

\end{itemize}

% \begin{figure}[ht]
% \centering
% \begin{minipage}[b]{0.55\linewidth}
% \includegraphics[width=1\textwidth, trim={0cm, 0cm, 0cm, 0cm}]{figures/architecture.pdf}
% \vspace{0.5cm}
% \caption{An overview of the proposed algorithm (GIAS). GIAS optimizes a latent code $z$ and generative model parameters $w$ to reconstruct the data which matches the gradient.}\label{fig:architecture}
% \end{minipage}
% \quad
% \begin{minipage}[b]{0.4\linewidth}
% \includegraphics[width=1\textwidth]{figures/authors_v2.pdf}
%     \caption{
%         Results of the proposed gradient attack.
%         Images of the authors are reconstructed from gradients with a pretrained generative model.}
%     \label{fig:authors}
% \end{minipage}
% \end{figure}

\begin{figure}[t]
\centering
\includegraphics[width=0.6\textwidth]{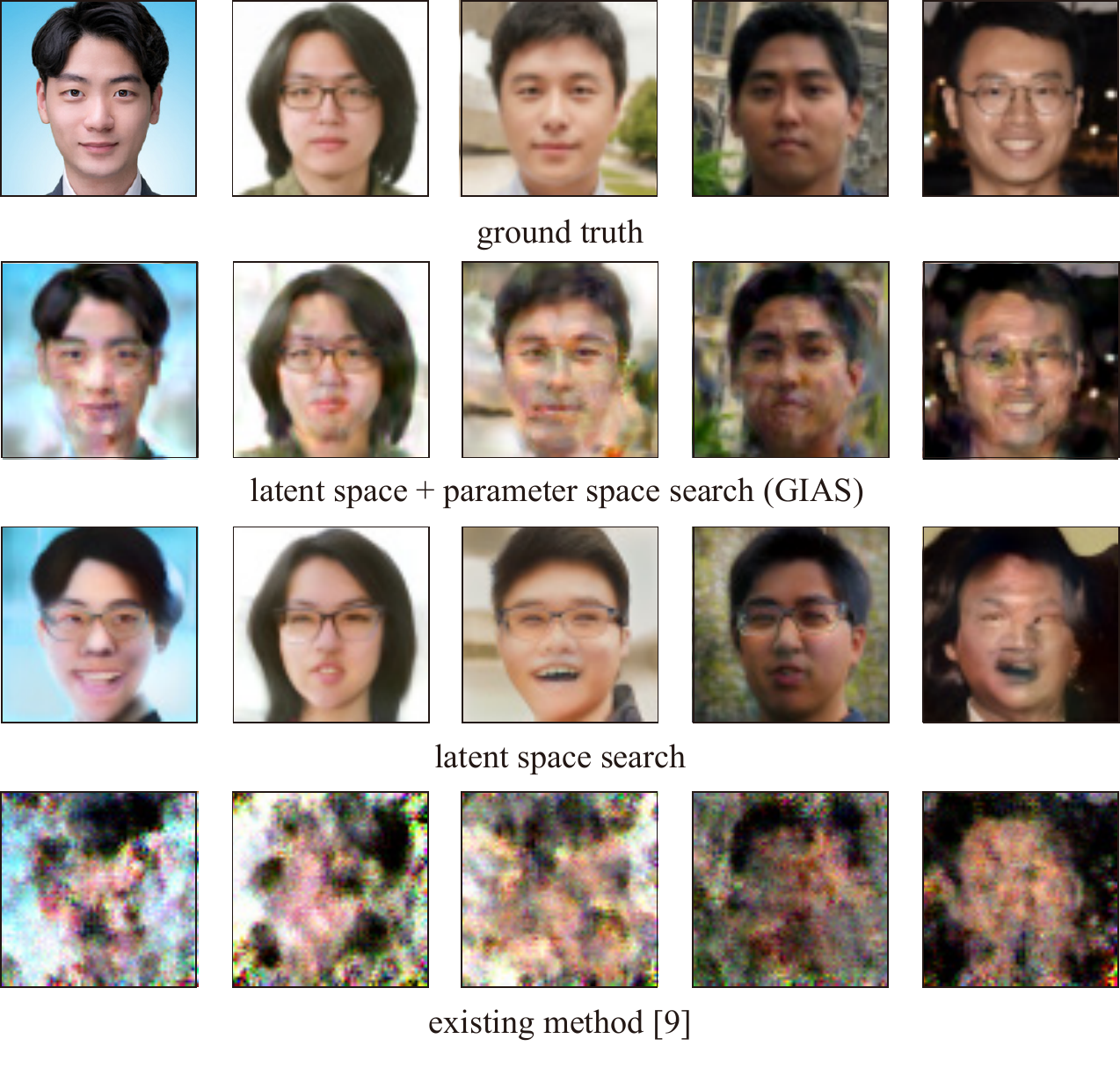}
\vspace{-0.3cm}
    \caption{
        \tblue{
        {\em An example showing the superiority of GIAS compared to existing method.}
        Images of the authors are reconstructed from gradients by exploiting a generative model pretrained on human face images.
        }
    }
    \label{fig:authors}
    \vspace{-0.5cm}
\end{figure}

\section{Related work}

    \paragraph{Privacy attacks in FL.}
    Early works \cite{Melis, shokri2017membership} 
    investigate
    membership inference from gradients
    to  check the possibility of privacy leakage in FL. 
    \citet{phong} demonstrate
    that it is possible to reconstruct detailed input image 
    when FL trains a shallow network such as single-layer perceptron.
        \citet{fan} and \citet{rgap} 
    consider a wider class of learning model
    and propose an analytical approach
       solving a sequence of linear systems 
       to reveal the output of each layer recursively.
    To study the limit of the gradient inversion
    in practical scenarios of 
    training deep networks via FL,
    a sequence of effort has been made
    formulating optimization problem
    to minimize discrepancy comparing
    gradients from true data and reconstructed data
        \cite{geiping, wang, nvidia, zhao, zhu}.
    
    \paragraph{Gradient inversion with prior.}
    The optimization-based approaches
    are particularly useful 
    as one can easily utilize prior knowledge by adding regularization terms, e.g.,
    total variation 
    \cite{wang, geiping}
    and BN statistics \cite{nvidia}, or changing discrepancy measure \cite{geiping} .
    % exploiting BN statistics, \citet{nvidia}
    % show that the inversion is possible for
    % % mini-batch size 48, although BN statistics is not always given
    % since FL can perform BN without exchanging BN statistics \cite{li2021fedbn, siloBN}.
    In \cite{nvidia}, 
    a privacy attack technique using 
    a generative model is introduced.  
    They however require a pretrained model,
    % or a set of user data examples to train one,
    while we propose a meta learning framework
    training generative model from gradients only. 
    In addition, our method of inverting gradient maximally exploit a given generative model by alternating search spaces, 
    which are analogous to the state-of-the-art GAN inversion techniques \cite{ generativePrior, Bau_2019_ICCV, zhu2020indomain}.
    
    \tblue{
        \paragraph{Generative model revealing private data.}
        Training a generative model with transmitted gradients 
        also demonstrates privacy leakage in FL.
        \citet{hitaj2017deep} introduce an algorithm to train a GAN
        regarding shared model in FL framework as a discriminator.
        % However, it was not exploiting gradient itself, 
        % and it requires real data of other class,
        % Trained generator can only generate representative image of a target class.
        \citet{wang} use reconstructed data from gradient to train a GAN. 
        % But the performance of the generative model is dependent on the quality of the reconstructed data,
        % since GAN is not directly learnt from gradient. 
        % and the trained generative model does not leverage the reconstruction algorithm.
        % Those works need an auxiliary dataset that contains real data to train a generative model
        % while we can train a model from only gradient via GIML.
        % The generator in \cite{hitaj2017deep, wang} are not exploiting gradient,
        % so the performance of a generative model is bounded 
        % to the shared model\cite{hitaj2017deep} or the reconstruction algorithm\cite{wang}.
        % Unlike \cite{hitaj2017deep, wang}, we can utilize obtained generative model to leverage reconstruction algorithm. 
        Those works require some auxiliary dataset given in advance to enable the training of GAN, 
        while we train a generative model using transmitted gradients only.
        Also, we not only train a generative model but also utilize it for reconstruction,
        while the generative models in \cite{hitaj2017deep, wang} are not used for the reconstruction.
        Hence, in our approach, the generative model and reconstruction can be improved interactively to each other as shown in Figure~\ref{fig:trained_sampled}.
        In addition, \cite{wang} is less sample-efficient than ours in a sense that they use gradients to reconstruct images and then train a generative model with the reconstructed images, i.e., if the reconstruction fails, then the corresponding update of the generative model fails too, whereas we train the generative model directly from gradients.
        
    }

\section{Problem formulation}
\label{sec:model}

In this section, we formally describe the gradient inversion (GI) problem.
    Consider a standard supervised learning for classification, 
    which optimizes neural network model $f_\theta$ parameterized by $\theta$ as follows:
    \begin{equation}
        \min_\theta ~ \sum_{(x,y) \in \set{D}} 
         \ell ( f_\theta( x ), y) \;,
    \end{equation}
    where $\ell$ is a point-wise loss function and 
    %the empirical expectation is taken over a labeled dataset
    $\set{D}$ is a dataset of input $x \in \mathbb{R}^{m}$ and label $y \in \{0,1\}^L$ (one-hot vector).
    In federated learning framework, each node reports
    the gradient of $\ell( f_\theta( x ), y)$ for sampled data $(x,y)$'s instead of directly transferring the data.
    The problem of inverting gradient is to reconstruct
    the sampled data used to compute the reported gradient.
    Specifically, when a node computes the gradient $g$ using a batch $\{(x^*_1,y^*_1), ..., (x^*_B,y^*_B) \}$, i.e.,
    $g = \frac{1}{B}
        \sum_{j =1}^B \nabla\ell ( f_\theta( x^*_j ), y^*_j)$,
    we consider the following problem of inverting gradient: % gradient inversion is formulated as the following optimization:
      \vspace{0.1cm}
      \begin{equation} \label{eq:GIP}
        %\min_{(x_j, y_j) \in \mathbb{R}^m \times \{0,1\}^L : j \in [B] } 
        \min_{\substack{(x_1, y_1), \cdots,  (x_B, y_B) \\ \in \mathbb{R}^m \times \{0,1\}^L}}  
        d\left(
        \frac{1}{B}
        \sum_{j =1}^B \nabla\ell ( f_\theta( x_j ), y_j),
        g \right)
        % \frac{1}{B}
        % \sum_{j =1}^B \nabla\mathcal{L} ( f_\theta( x ), y),
        % \frac{1}{B}
        % \sum_{(x^*, y^*) \in \set{B}^*} \nabla\mathcal{L} ( f_\theta( x^* ), y^*) \right)
        \;,
    \end{equation}
    \vspace{0.1cm}
    where $d(\cdot, \cdot)$ is a measure of the discrepancy between two gradient, e.g., $\ell_2$-distance \cite{zhu, nvidia} or negative cosine similarity \cite{geiping}.
    It is known that label $y$ can be almost accurately recovered by simple methods just observing the gradient at the last layer \cite{zhao, nvidia},
    while reconstructing input $x$ remains still challenging as it is often under-determined even when the true label is given.
    For simplicity, we hence focus on the following minimization to reveal the inputs from the gradient given the true labels:
    \vspace{-0.1cm}
    \begin{equation} \label{eq:GIP-sim}
        \min_{x_1, ..., x_B \in \mathbb{R}^m }  
     c\left(x_1, ..., x_B; \theta, g % d; \theta, \nabla^* 
     %\set{B}^*
     \right)
        \;,
    \end{equation}
    where %with slight abuse of notation, 
    % we denote
    % by $g=\frac{1}{B}
    %     \sum_{j =1}^B \nabla\ell ( f_\theta( x^*_j ), y^*_j)$
    %     the target gradient,
    %     and 
        we denote
    by $c\left(x_1, ..., x_B; \theta, g
    \right)$ the cost function in \eqref{eq:GIP} given $y_j = y^*_j$ for each $j = 1, ..., B$.
    
    %$\set{B}^*$ is the true batch used to compute the gradient given model parameter $\theta$.

    % More formally,
    %       \begin{equation} \label{eq:GIP}
    %     \min_{x_j \in \mathbb{R}^m : j \in [B]}  
    %     \set{D}\left(
    %     \sum_{j \in [B]} \nabla\mathcal{L} ( f_\theta( x_j ), y^*_j),
    %     \sum_{j \in [B]} \nabla\mathcal{L} ( f_\theta( x^*_j ), y^*_j) \right) \;,
    % \end{equation}

         %A number of algorithms have been proposed to easily recover label $y$ accurately \cite{XX, YY}, while reconstructing input $x$ is challenging.
     
    %even for the regression task.?
        
\section{Methods}
\label{sec:analysis}
\label{sec:method}

    %    We demonstrated gradient inversion is much easier with image prior such as GAN.
    The key challenge of inverting gradient
    is that solving \eqref{eq:GIP} is often under-determined, i.e., a gradient contains only insufficient information to recover data.
    Such an issue is observed even when 
    the dimension of gradient is much larger than that of input data.
    Indeed, 
    \citet{rgap} show that there exist a pair of different data having 
    the same gradient, so called twin data, even when the learning model is large.
    % %The gradient inversion problem in \eqref{eq:GIP} is 
    % is that % the problem of reconstructing $x_j$'s in \eqref{eq:GIP-sim}
    % it is often under-determined, in particular, when 
    % a gradient of model size is given to reconstruct a large batch.
    To alleviate this issue, % such an under-determined issue,
    a set of prior knowledge on the nature of data 
    can be considered. 
    % has been utilized in various forms.
    %which are elaborated in what follows.
    % In particular, in case of recovering images, the following prior information can be considered:
        %\paragraph{Total variation~\cite{geiping}} 

    When inverting images, \citet{geiping} propose to add 
    the total variation regularization $R_{\text{TV}}(x)$ %for each image $x$ 
    to the cost function in \eqref{eq:GIP-sim}
    since neighboring pixels of natural images are likely to have similar values. 
    More formally, 
    \begin{align}
    R_{\text{TV}}(x):= \sum_{(i, j)} \sum_{(i', j') \in \partial (i, j)} \| x(i,j) - x(i',j')\|^2 \;,
    \end{align}
    where $\partial(i,j)$ is the set of neighbors of $(i,j)$.
    This method is limited to the natural image data.

    For general type of data, one can consider exploiting
    the batch normalization (BN) statistics from nodes.
    This is available in the case that 
    the server wants to utilize batch normalization (BN) in FL,
    and thus collects the BN statistics (mean and variance) of batch
    from each node, in addition, with every gradient report \cite{fedavg}. 
        %For this specific case, 
    To be specific, 
        \citet{nvidia}
        propose to 
        % the attacker can exploit
        % the additional information 
        employ the regularizer $R_{\text{BN}}(x_{1}, ..., x_B; \theta)$
        which quantifies the discrepancy between
        the BN statistics of estimated $x_j$'s and those of true $x^*_j$'s on each layer of the learning model. More formally,
        \begin{align*}
        \!R_{\text{BN}}(x_{1}, ..., x_B; \theta)
        \!:=\! \sum_{l} \|\mu_l 
        \!-\! \mu_{l, \text{exact}}
        %\text{BN}_l (\text{mean})
        \|_2
        +\|\sigma^2_l %(x_{1}, ..., x_B) - 
        \!-\! \sigma^2_{l, \text{exact}}
        %\text{BN}_l (\text{variance})
        \|_2 ,
        \end{align*}
        where 
        $\mu_l(x_{1}, ..., x_B; \theta)$
        and $\sigma^2_l (x_{1}, ..., x_B; \theta)$
        (resp. $\mu_{l, \text{exact}}(x^*_{1}, ..., x^*_B; \theta)$
        and $\sigma^2_{l, \text{exact}}(x^*_{1}, ..., x^*_B; \theta)$)
        are the mean and variance of 
        $l$-th layer feature maps 
        for the estimated batch $x_{1}, ..., x_B$ 
        (resp. the true batch $x^*_{1}, ..., x^*_B$ ) given $\theta$.
        This is available only if clients agree 
        to report their exact BN statistics at every round.
        But not every FL framework report BN statistics \cite{li2021fedbn, siloBN}.
        In that case, \citet{nvidia} also propose to use 
        the BN statistics over the entire data distribution
        as a proxy of the true BN statistics,
        and reports that
        the gain from the approximated BN statistics
        is comparable to that from the exact ones.
        The applicability of $R_{\text{BN}}$
        with the approximated BN statistics
        is still limited as
        the proxy needs to be additionally recomputed over the entire data distribution
        at every change of $\theta$. However, 
        %the significant improvement 
        this demonstrates the significant impact 
        of knowing 
        the data distribution in the gradient inversion
        and motivates our methods using and learning a generative model on the user data, described in what follows.

     \subsection{Gradient inversion with trained generative model}
       \label{sec:GIAS}
       
\begin{figure}[ht]
\centering
% \begin{minipage}[b]{0.55\linewidth}
\includegraphics[width=0.8\textwidth, trim={0cm, 0cm, 0cm, 0cm}]{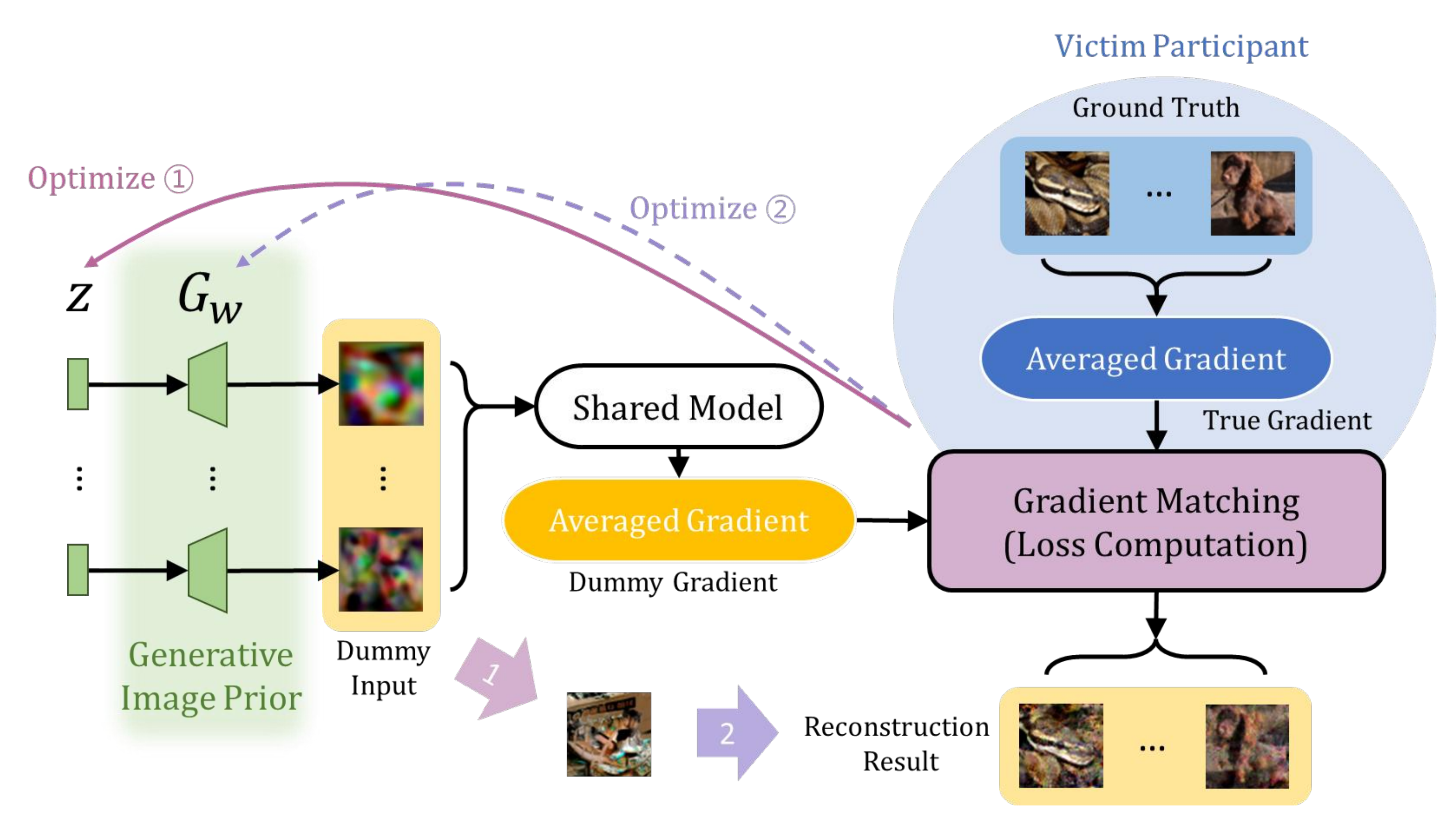}
\vspace{0.5cm}
\caption{{\em An overview of GIAS}. GIAS optimizes a latent code $z$ and generative model parameters $w$ to reconstruct the data which matches the gradient.}\label{fig:architecture}
% \end{minipage}
\end{figure}

    % \input{algorithms/GIAS}

        %\paragraph{Generative model and deep image prior} 
Consider a decent generative model
    $G_w: \mathbb{R}^k  \mapsto \mathbb{R}^m$
    trained
    on the approximate (possibly exact) distribution of user data $\set{D}$
    such that
%    $k \le m$ and 
    %$\min_z \|G_w (z) - x^*\| \ll 1$
    $x^* \approx G_w(z^*)$
    for $(x^*, \cdot) \in \set{D}$
    and its latent code $z^* = \argmin_z \|G_w (z) - x^*\|$.
    To fully utilize such a pretrained generative model, we propose
     gradient inversion in alternative spaces
     (GIAS), of which pseudocode
     is presented in Appendix~\ref{sec:algorithm},
     %described in Algorithm~\ref{algo:GIAS}, 
     which performs latent space search over $z$ and then 
     parameter space search over $w$. 
     We also illustrate the overall procedure of GIAS in Figure~\ref{fig:architecture}.
     %of which are elaborated in what follows.

    \paragraph{Latent space search.}  
    Note that the latent space is typically much smaller than 
    the ambient input space, i.e., $k \ll m$, for instances, DCGAN \cite{dcgan} of $k = 100$ and StyleGAN \cite{Karras_2019_CVPR} of $k = 512 \times 16$  for image data of $m =
    (\text{width}) \times (\text{height}) \times (\text{color})$ such as 
    $32 \times 32 \times 3$, $256 \times 256 \times 3$, or larger.
    Using such a pretrained generative model with $k \ll m$,
    the under-determined issues
    of \eqref{eq:GIP-sim}
    can be directly mitigated 
    by narrowing down the searching space from 
    $\mathbb{R}^m$ to 
    $\{G_w(z) : z \in \mathbb{R}^k\}$.
    % since the number of variables to be recovered
    % is reduced from $m$ to $k$ while the constrains are the same.
    %This can be formally written as follows:
    Hence, GIAS first performs the latent space search in the followings:
    % More formally,
    % we consider the following variant of 
    % the gradient inversion with generative model:
    \begin{equation} \label{eq:G-GIP-sim}
        \min_{%z_{1:B} := 
        z_1, ..., z_B \in \mathbb{R}^k} ~  % y \in {0,1} ?
        c\left(G_w(z_1), ..., G_w(z_B) \right) \;.
    \end{equation}
    %we replace $x_j$ with $G_w(z_j)$ and optimize over latent space.
    
    % Since even the state of the art generative models
    % have non-zero approximation error
    % $\min_z \|G_w (z) - x^* \| > 0$ for some data $\set{D}$
    
    % Even when 
    
    Considering a canonical class of neural network model $f_\theta$,
    we can show that 
    the reconstruction of $x^*$ by latent space search 
    in \eqref{eq:G-GIP-sim}
    aligns with that by input space search in \eqref{eq:GIP-sim}
    if the generative model $G_w$ approximates input data %$\set{D}$
    with small enough error.
      \begin{property} \label{prop:conti}
      %Suppose that
      For an input data $x^* \in [0,1 ]^m$ %\mathbb{R}^m$, 
      consider
      the gradient inversion problem of minimizing cost $c$
      in \eqref{eq:GIP-sim},
      where
      a canonical form of deep learning
      for classification is considered
      and 
      the discrepancy measure $d$ is $\ell_2$-distance. 
    %   %learning model
    %   the learning model $f_\theta$ is 
    %   a standard form of neural network 
    %   (such as multi-layer perceptron or convolutional neural network)
    %   with C1 ({continuously differentiable}) activations
    %   (e.g., sigmoid and exponential linear), 
    %   loss function $\ell$ is C1  (e.g., logistic and exponential),
    %   and the discrepancy measure $d$ is $\ell_2$-distance. 
%       $\ell_2$-distance
    Suppose that it has the unique global minimizer at $x^*$.
      Let $\varepsilon \ge 0$ 
      be the approximation error bound
      on $x^*$ for generative model $G_w:
       [0,1 ]^k \mapsto  [0,1 ]^m$
    %   \mathbb{R}^k \mapsto \mathbb{R}^m$
    %   with $k \le m$.
      %, i.e.,
      %$\min_{z \in \mathbb{R}^k} \|x^* -  G_w(z)\| \le  \varepsilon$.
    Then, there exists $\delta(\varepsilon) \ge 0$ such that for any $z^* \in \argmin_{z} c(G_w(z))$, % verifies
        \vspace{0.05cm}
        \begin{align}
        \| G_w(z^*) - x^* \| \le \delta(\varepsilon) \;,
        \end{align}
        of which upper bound $\delta(\varepsilon) \to 0$ as $\varepsilon \to 0$.
      \end{property}
      A rigorous statement of Property~\ref{prop:conti} and its proof are provided in Appendix~\ref{sec:property}, where we prove and use that the cost function 
      is continuous around $x^*$ under the assumptions.
      % We note that when using other class of $f_\theta$,  
    %   the above property can be untrue
      This property justifies solving the latent space search in \eqref{eq:G-GIP-sim} for FL scenarios training neural network model
      while it requires an accurate generative model.
    % Even if the generative model $G_w$ approximates input data $\set{D}$
    % with small error, in general,
    % the reconstruction of $x^*$ by latent space search 
    % in \eqref{eq:G-GIP-sim}
    % can be very far from that by input space search in \eqref{eq:GIP-sim}

\paragraph{Parameter space search.}
    
    % we often observe that 
    % we cannot find the images which is not in the range of generative model
    % since generative model does not express the whole image space \cite{generativePrior}.
    %
    
    Using the latent space search only, 
    there can be inevitable reconstruction error
    due to the imperfection of generative model.
    This is mainly because we cannot perfectly prepare the generative model for every plausible data in advance. 
    Similar difficulty of the latent space search
    has been reported even when inverting GAN
    \cite{zhu2020indomain, generativePrior, Bau_2019_ICCV}
    for plausible but new data directly, i.e., 
    $\min_z \|G_w(z) - x^*\|$ given $x^*$,
    rather than inverting gradient.
    % even when inverting GAN, i.e., $\min_z \|G_w(z) - x^*\|$ given $x^*$ directly,
    % rather than inverting gradient.
    %$\|G_w(z^*) - x^*\|$ with 
    % fail at obtaining     
    % latent code $z^*$
    % %with  large error $\|G_w(z^*) - x^*\|$    
    % Despite of remarkable advances in generative models,
    % Indeed, such an issue of poor $z^*$
    % for plausible but new image $x^*$
    % has been reported in researches
    % on .
    % have well understood the issue of poor $z^*$ for 
    % plausible but new image $x^*$ out of training dataset,
    \citet{generativePrior} propose an instance-specific model adaptation,
    which slightly adjusts the model parameter $w$ to (a part of source image) $x^*$ after obtaining a latent code $z^*$ for $x^*$. 
    Inspired by such an instance-specific adaptation, %in GIAS, 
    GIAS performs
    the following parameter space search
    over $w$ preceded by the latent space search over~  $z$:
        \begin{equation} \label{eq:G-GIP-model}
        \min_{%z_{1:B} := 
        w_1, ..., w_B} ~  % y \in {0,1} ?
        c\left(G_{w_1}(z_1), ..., G_{w_B}(z_B) \right) \;, 
    \end{equation}
    where $z_1, \dots, z_B$ are obtained from \eqref{eq:G-GIP-sim}.

    \paragraph{Remark.}
    
    % See (R-GAP+generative model) details in Appendix.

    % performs the latent code optimization in \eqref{eq:G-GIP-sim} 
    % and then the parameter model optimization in \eqref{eq:G-GIP-model}.

    % Assuming $G_w$ is pretrained well, 
    % we anticipate to find a small change
    % in $w$ in the optimization over $w$ followed by 
    % that over $z$. %  done by only small change in $w$  already.
    We propose the optimization over $w$ followed by
    that over $z$ sequentially % , as described in Algorithm~\ref{algo:GIAS}. 
    This is to 
    maximally utilize the benefit of mitigating the under-determined issue from reducing the searching space 
    %of low dimension $k \ll m$ % in \eqref{eq:G-GIP-sim}
    on the pretrained model.
    However, the benefit would be degenerated 
    %or disappear
    if $z$ and $w$ are optimized jointly or $w$ is optimized first. %A numerical comparison in 
    We provide an empirical justification
    on the proposed searching strategy in Section~\ref{sec:ablation}.
%    with 
 %   Figure~\ref{figure:images}. % empirically justifies our method.
    
    % the change in $w$ may be significant
    % and thus lose 
    %such a benefit is 
    % Indeed, \citet{generativePrior}
    % reported that 
    % given $w$ after $z$, the original semantic representations of the pretrained model.
    
    %$G_w$ are mostly preserved.

    %     Features found by $z$ also preserved with respect to $w$.
    % However, for another algorithm optimizing $z$ followed by $x$,
    % such features are discarded and the algorithm becomes Inverting Gradients with better initialization.
    % Figure~\ref{figure:images} shows the results, optimizing $z$ is slightly better than Inverting Gradients.

    % rather than optimizing $z$ after $w$
    % or $z$ and $w$ jointly
    % since it 
    % This specific sequential search can 
    % assuming
    % optimized parameter $w'$ is close to original parameter $w$, as $w$ is pretrained generative model.
        
    % optimization of $z$ followed by optimizing $w$ 
    % rather than joint
    % is from the assumption that
    % optimized parameter $w'$ is close to original parameter $w$, as $w$ is pretrained generative model.
    % If $z$ and $w$ is optimized jointly or $w$ is optimized first, distance between $w$ and $w'$ will increase
    % and benefits from using pretrained model disappears.

    %Algorithm~\ref{algo:GIAS} describes 
    %the proposed method, in which 
    We perform each search in GIAS
    using a standard gradient method to the cost function directly. 
    It is worth noting that those optimizations
    \eqref{eq:G-GIP-sim} and \eqref{eq:G-GIP-model}
    with generative model
    can be tackled in a recursive manner
    as R-GAP \cite{rgap}
    %solves a sequence of linear programming 
    reconstructs each layer from output to input. 
    % We provide how to perform 
    % the recursive reconstruction with generative model
    We provide details and performance of the recursive procedure in Appendix~\ref{sec:rgap}, where employing generative model improves the inversion accuracy of R-GAP substantially, while R-GAP apparently suffers from an error accumulation issue when $f_\theta$ is a deep neural network.

\subsection{Gradient inversion to meta-learn generative model}\label{sec:2.2}

% \paragraph{\tblue{Learning a generative model from multiple gradients}}
    % GIAS requires a decent generative model, which
    % is often unavailable in practice.
    %However, pretrained models are not always exist.
    % In worse case, information about which pretrained models to be used, are not known.
    For the case that pretrained 
    generative model is  unavailable, 
    we devise an algorithm
    to train a generative model $G_w$
    for a set $\set{S} = \{(\theta_i, g_i)\}$ of gradient inversion tasks.
    %described by $(\theta_i, g_i)$. % to learn individual generative model per data.
    Since each inversion task can be considered as a small learning task
    to adapt generative model per data, 
    we hence call it gradient inversions to meta-learn (GIML).
    % since each gradient inversion task learns an individual generative model per data,
    % while the individual models are aggregated in a generative model as meta information shared over the set of inversion tasks.
    The detailed procedure of GIML is presented in Appendix~\ref{sec:algorithm}.
    % Algorithm~\ref{algo:gen-meta}, and elaborated in what follows. 
    We start with an arbitrary initialization of $w$,
    and iteratively update
    toward $w'$ from
    a variant of GIAS for 
    $N$ tasks sub-sampled from $\set{S}$,
    which 
    %The variant GIAS  %for $N$ tasks 
    is different than multiple applications of GIAS for each task
    in two folds: (i) $\ell_2$-regularization in latent space search;
    and (ii) an integrated optimization on model parameter.
    The variant first finds optimal latent codes $\vec{z}^*_i = (z_{i1}^*, ..., z_{iB}^*)$
    for each task $i$
    %$z^*_{ij}$'s for each $i= 1,\dots, N$ and $j= 1,\dots, B$ given $w$ 
    with respect to the same cost function of 
    GIAS but additional $\ell_2$-regularization.
    Note that the latent space search with untrained generative model easily diverges.
    The $\ell_2$-regularization is added to prevent the divergence of $\vec{z}^*_i$.
    %regularization $z$ norm regularizer was added. 
    %The regularization prevents 
    % $z$ should be regularized and distribution of $z$ will be eventually close to the distribution of typical generative model, $\mathcal{N}(0,1)$.
    Once we obtained $\vec{z}^*_i$'s, 
    $w'$ is computed by
     few steps of gradient descents
    for an integrated parameter search 
    to minimize 
    $\sum_i c(G_{w'} (z^*_{i1}), ..., G_{w'} (z^*_{iB}); \theta_i, g_i)$.
    This is because in GIML, 
    we want meta information $w$ to help  GIAS for each task rather than
    solving individual tasks, while
    after performing GIML to train $w$,
    we perform GIAS to invert gradient with the trained $w$. 
    This is analogous to the Reptile in \cite{nichol2018firstorder}.

\section{Experiments}
\label{sec:exp}

    \begin{figure}[!t]
    \vskip 0.2in
      \centering
      \begin{subfigure}[b]{0.57\textwidth}
        \centering
        \includegraphics[width=1\textwidth, trim={1.2cm, -1.0cm, 0cm, 2cm}]{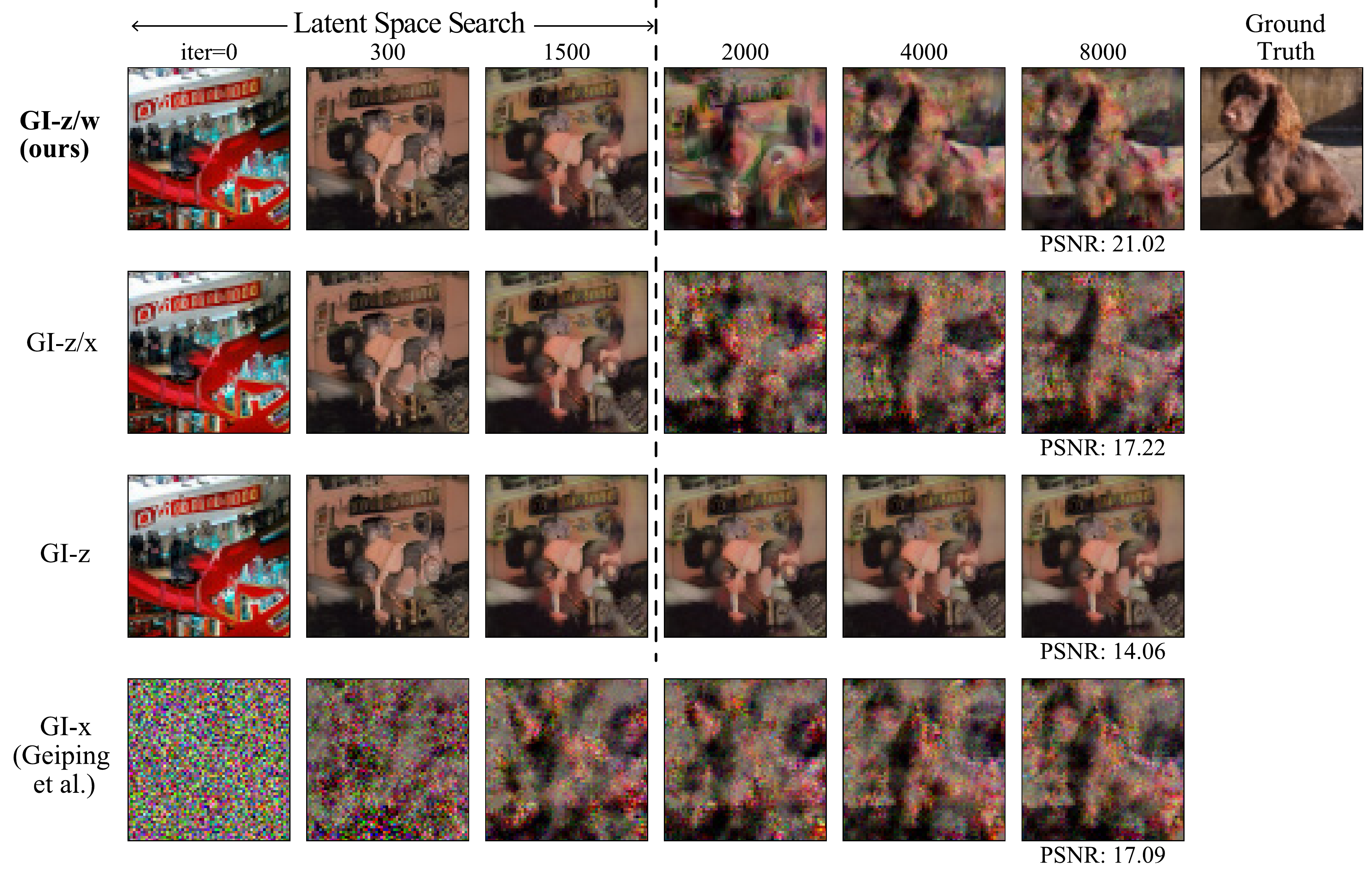}
        \caption{}
        \label{figure:images1}
      \end{subfigure}
      \begin{subfigure}[b]{0.41\textwidth}
        \centering
        \includegraphics[width=1\textwidth, trim={1.0cm, 0.5cm, 0.0cm, -2.0cm}]{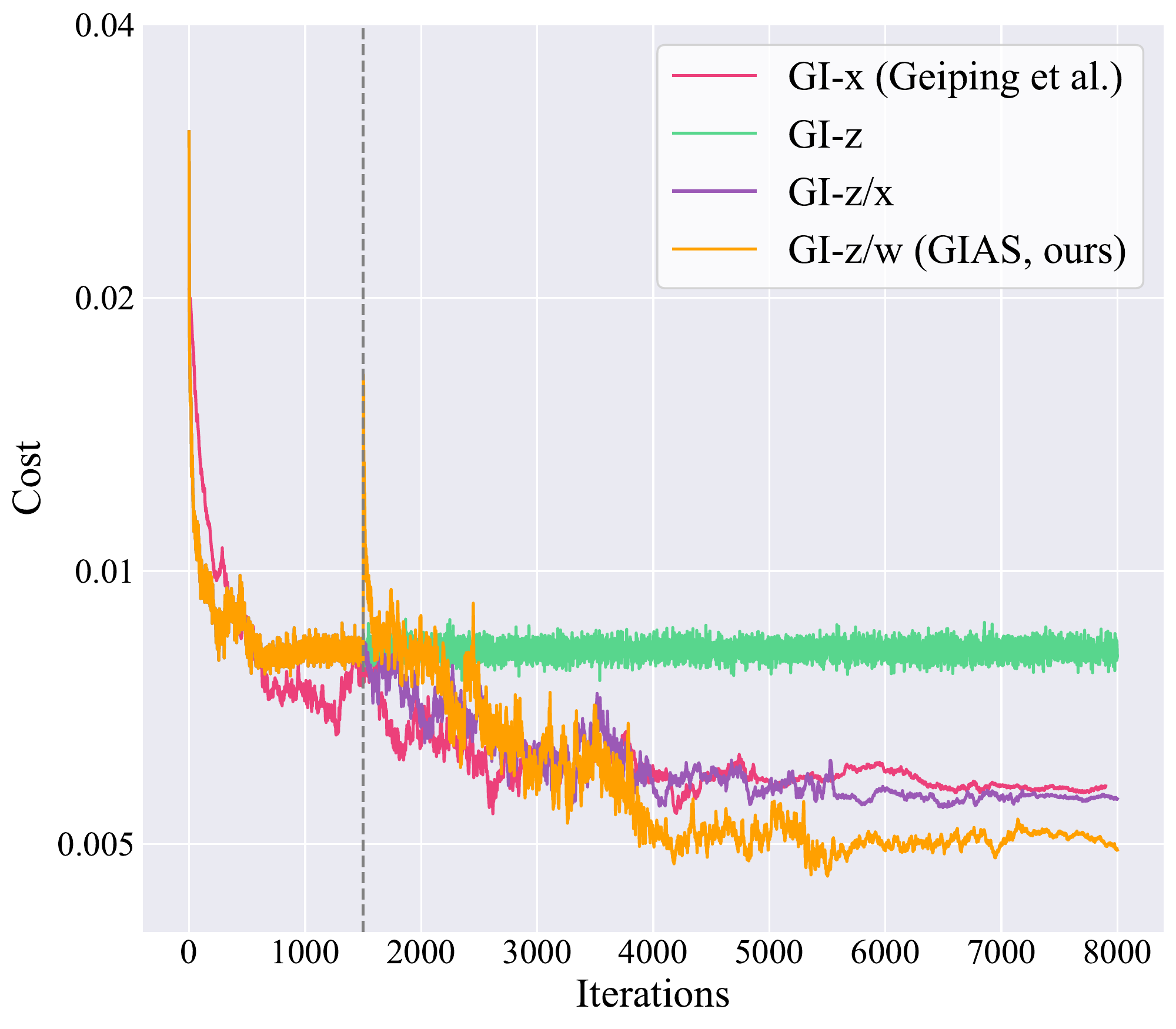}
        \caption{}
        \label{figure:images2}
      \end{subfigure}
      \caption{
            {\em Comparison of different searching spaces.}
            (a) 
            Each row shows reconstructed images of different optimization domains.
            The first three rows share the same latent space search of $1,500$ iterations, and  GI-$z/w$ is verified to be the best option to fully exploits the knowledge inside the generative model.
            (b) 
                Cost function over iterations of different optimization domains. 
      }
      \label{figure:images}
    %\vskip -0.2in  
    \end{figure} 

    \begin{figure}[t]
      \centering
      \begin{subfigure}[b]{0.33\textwidth}
         \centering
        \includegraphics[width=\textwidth, trim={1.2cm, 0cm, 2cm, 2cm}]{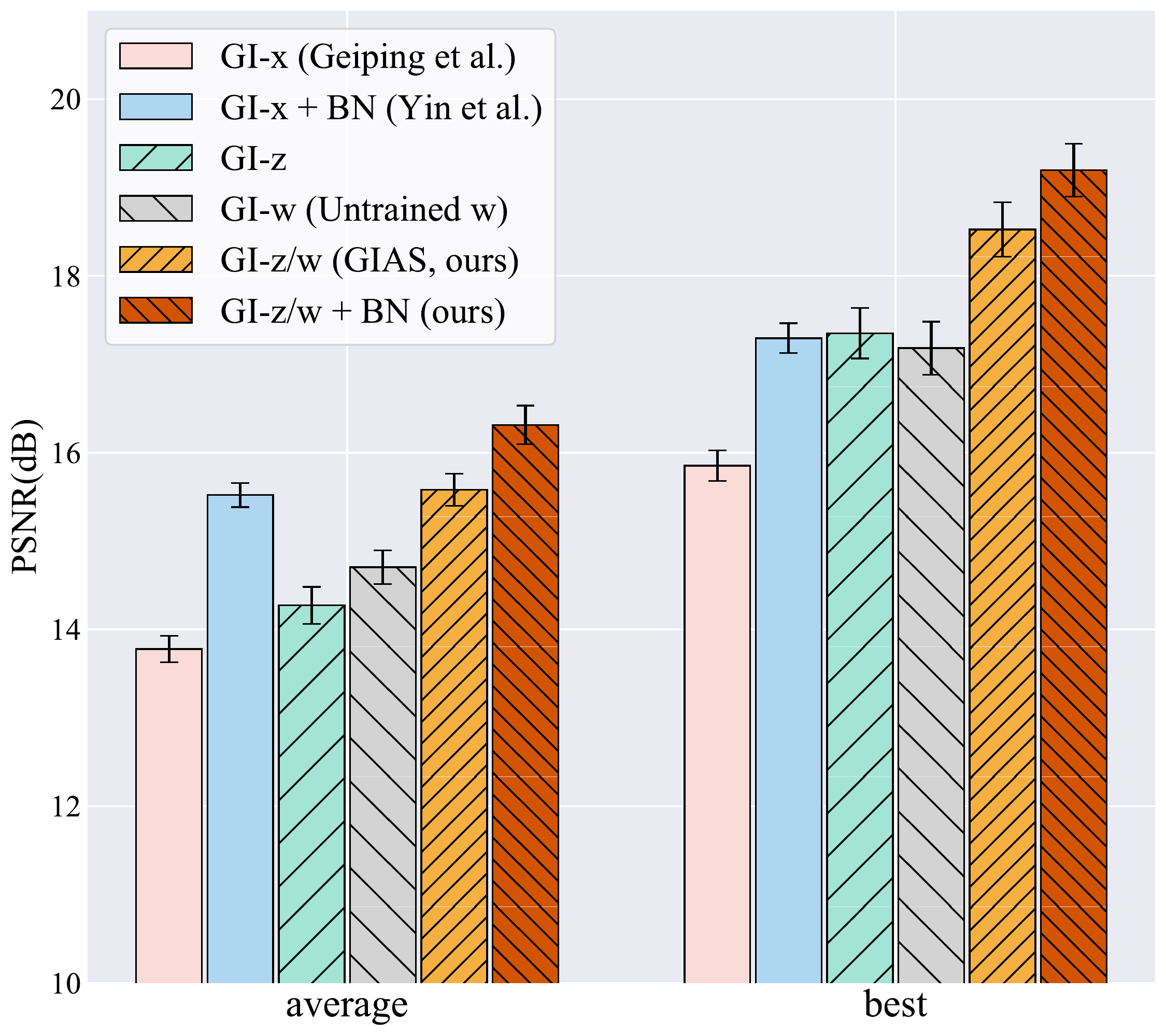}
        \caption{}
        \label{figure:ablation_graph}
      \end{subfigure}
      \hfill
      \begin{subfigure}[b]{0.66\textwidth}
         \centering
        \includegraphics[width=\textwidth, trim={2cm, 1cm, 3cm, 2cm}]{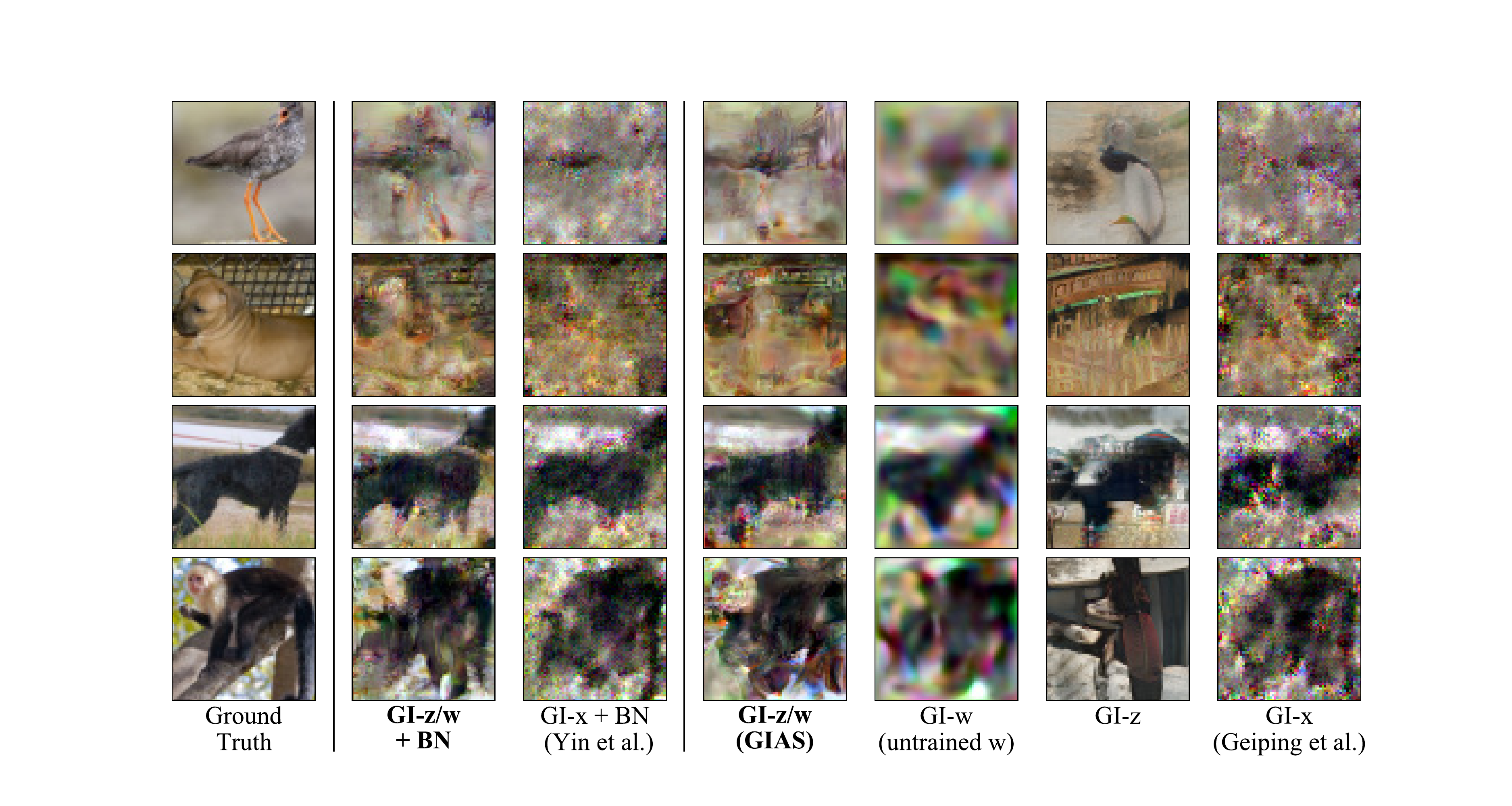}
        \caption{}
        \label{figure:ablation_image}
      \end{subfigure}
      
      \caption{
      \tblue{
            {\em Comparison of \sota models and ours}.
            Replacing GI-$x$ with GI-$z/w$ (GIAS) 
            regardless of using BN \cite{nvidia} or not \cite{geiping} 
            provides substantial improvement in the reconstruction accuracy. 
          (a)
            Average PSNR and best PSNR in a batch throughout the experiments.
            % is shown with 95\% confidence interval.\
          (b)
            An ablation study and comparison of reconstruction results with our models and \sota models. 
            We highlight the proposed models in {\bf bold}.
      }
      }
      \label{figure:ablation}
    \end{figure}

% \subsection{Experiment settings}
% \label{sec:exp-setting}

    \paragraph{Setup.} Unless stated otherwise, 
    we consider the image classification task 
    on the validation set of ImageNet~\cite{imagenet} dataset 
    scaled down to $64 \times 64$ pixels (for computational tractability) 
    and use a randomly initialized ResNet18~\cite{he2016deep} for training. 
    %, we rescaled the ImageNet images to 64x64. 
    % Recalling the under-determined issue
    % is the major challenge in gradient inversion, 
    % deeper and wider $f_\theta$ makes the gradient inversion easier \cite{geiping}.
    % Hence, the choice of ResNet18 as learning model is the most difficult setting within ResNet architectures since it contains the least information. 
    %than shallower and narrower ones.
    For deep generative models in GIAS, we use  StyleGAN2~\cite{Karras2019stylegan2} trained on ImageNet. 
    We use a batch size of $B = 4$ as default 
    and use the negative cosine to measure  the gradient dissimilarity $d(\cdot,\cdot)$.
    % which apparently provides better inversion
    % performance than $\ell_2$-distance in general \cite{geiping}. 
    We present detailed setup in Appendix~\ref{sec:exp-setting}.
    Our experiment code is available at \url{https://github.com/ml-postech/gradient-inversion-generative-image-prior}.
    
    \vspace{-0.2cm}

    \paragraph{Algorithms.} We evaluate several algorithms
    for the gradient inversion (GI) task in \eqref{eq:GIP-sim}. They differ mainly in which spaces each algorithm searches over: the input $x$, the latent code $z$, and/or the model parameter $w$.
    %Unless stated otherwise, each algorithm 
    Each algorithm is denoted by GI-($\cdot$), 
    where the suffix indicates the search space(s).
    For instances, 
    GI-$z/w$ is identical to the proposed method, GIAS,
    % in which the latent code searching preceded 
    % by the model parameter searching,
    and GI-$x$ is the one proposed by \citet{geiping}.
% \blfootnote{The code is available at \url{https://github.com/ml-postech/gradient-inversion-generative-image-prior}}

    %     \begin{itemize}
    %         \item \textbf{GI-$\boldsymbol{z/w}$ (GIAS)} : latent space search + parameter space search (ours)
    %         \item GI-$z$ : latent space search only (ours)
            
    %         \item GI-$w$ : parameter space search only (ours, untrained generative model)
    %         \item GI-$z/x$ : latent space search + image space search (ours)
    %         \item GI-$x$: Inverting Gradients \cite{geiping}
    %         \item GI-$x$ + BN: 
    %         This is identical to the algorithm proposed
    %         by \citet{nvidia}, in which u
    %         \item GI-$z/w$ + BN : See through gradients \cite{nvidia} + latent space search + parameter space search (ours)
            
    %     \end{itemize}

    \subsection{Justification of GIAS design}
    \label{sec:ablation}
     We first provide an empirical justification 
     of the specific order of searching spaces in GIAS (corresponding to GI-$z/w$) 
     to fully utilize a pretrained generative model. 
     To do so, we provide Figure~\ref{figure:ablation_image} comparing algorithms with different searching spaces: GI-$z/w$, GI-$z/x$, GI-$z$, and GI-$x$, of which the first three share the same latent space search
     over $z$ for the first $1,500$ iterations. 
    %  As GI-$w/z$ is meaningless because optimizing $w$ does not preserve $z$, 
        As shown in Figure~\ref{figure:images}(a), 
        %comparing GI-$z$ to GI-$x$,
        the latent space search over $z$ 
        quickly finds plausible image in a much shorter number of iterations
        than GI-$x$, while it does not improve after a certain point due to the imperfection of pretrained generative model.
        Such a limitation of GI-$z$ is also captured in Figure~\ref{figure:images}(b), where the cost function of
        GI-$z$ is not decreasing after a certain number of optimization steps. To further minimize the cost function, 
        one alternative to GI-$z/w$ (GIAS) is GI-$z/x$, 
        which can further reduce the loss function whereas
        the parameter search in GI-$z/w$ 
        seems to provide more natural reconstruction of the image
        than GI-$z/x$. The superiority of GI-$z/w$
        over GI-$z/x$ may come from 
        that the parameter space search 
        exploits an implicit bias from optimizing a good architecture 
        for expressing images, {c.f.}, deep image prior \cite{dip}.
        In Appendix~\ref{sec:ffhq_prior} and Figure~\ref{fig:authors}, we also present the same comparison 
        on FFHQ (human-face images) \cite{Karras_2019_CVPR}
        where diversity is much smaller than that of ImageNet. 
        On such a less diverse dataset, the distribution can be easily learned, and the gain from training a generative model is larger. 
        
        % As shown in Figure~\ref{figure:images}(a), 
        % By searching only latent space, GI-$z$ finds smooth images, 
        % but often perceptually far from original images. 
        % \tblue{We note that once we're able to narrow down
        % the distribution of data 
        % to a smaller one such as FFHQ (human-face images) \cite{Karras_2019_CVPR},
        % then the gain from generative model increases
        % and thus GI-$z$ provide better reconstructions than GI-$x$.
        % We report the results on FFHQ in Appendix.}
        % % With FFHQ dataset , face images aligned with face landmarks, 
        % % GI-$z$ shows better results, and the results are in Appendix. 
        % This limitation of GI-$z$ is also captured in Figure~\ref{figure:ablation_image}(b), where the cost function 
        % is not decreasing after some optimization steps.
        % To further minimize the cost function, 
        % one alternative to GI-$z/w$ (GIAS) is GI-$z/x$, 
        % which can further reduce the loss function but
        % is not better than GI-$z/w$.
        % This may be because that the parameter space search 
        % exploits an implicit bias from optimizing a good architecture 
        % for expressing images, \emph{c.f.}, deep image prior \cite{dip}.
        
        % Searching parameter space only is equal to 

        % Figure~\ref{figure:images} compares results of GI-$z/x$, GI-$z$, and GI-$z/w$.
        
        % optimizing $z$ followed by optimizing $x$ is no better than optimizing $x$ directly with better initialization.
        
        % Searching parameter space only is equal to Deep image prior \cite{dip}.
        % The result images in Figure \ref{figure:ablation_image} look blurred, and LPIPS score is worse than Inverting Gradients.

\newcommand{\etal}{{et al}.\@ }

% \begin{table*}[!t]
% \caption{
%     Comparison of (our methods) with state-of-the-art methods. 
%     PSNR, SSIM, and LPIPS\cite{zhang2018perceptual} was used for evaluation.
% }
% \label{tab:against_sota}
% \centering
% \resizebox{0.5\linewidth}{!}{
% \begin{tabular}{lccc}
%     \toprule
%     \multirow{2}{*}{\textbf{Method}} &  \multicolumn{3}{c}{\textbf{Distance to Original Images}}\\
%     \cmidrule{2-4}
%     & PSNR $\uparrow$ & SSIM $\uparrow$ & LPIPS $\downarrow$  \\
    
%     \midrule
    
%     GI-$x$ %Geiping \etal NeurIPS'20
%     \cite{geiping}           & $13.78$ & $0.2542$ & $0.4376$ \\
%     GI-$z$ (ours)                              & $14.27$ & $0.3106$ & $0.3233$ \\
%     GI-$w$ (ours)                         & $14.70$ & $0.3519$ & $0.5121$ \\
%     GI-$z/w$ (GIAS) (ours)                               & $\mathbf{15.58}$ & $\mathbf{0.3895}$ & $\mathbf{0.3023}$ \\
    
%     \midrule
    
%     GI-$x$+BN %Yin \etal CVPR'21~
%     \cite{nvidia}                 & $15.52$ & $0.3513$ & $0.3645$ \\
%     GI-$z/w$+BN (ours)                             & $\mathbf{16.31}$ & $\mathbf{0.4311}$ & $\mathbf{0.2861}$ \\
    
%     \bottomrule
% \end{tabular}}
% \end{table*}

% % Please add the following required packages to your document preamble:

\begin{table*}[!t]

\caption{
    Comparison of our methods with state-of-the-art methods. 
    Adding our method makes performance improvement versus two baseline methods.
    PSNR, SSIM, and LPIPS\cite{zhang2018perceptual} are used to evaluate reconstruction results.
    We highlight the best performances in {\bf bold}.
}
\label{tab:against_sota}
\resizebox{\textwidth}{!}{
    \begin{tabular}{c||cccc|cc}
        \toprule
        {Method} & \multicolumn{1}{l}{GI-$x$     \cite{geiping}} & \multicolumn{1}{l}{GI-$z$ (ours)} & \multicolumn{1}{l}{GI-$w$ (ours)} & \multicolumn{1}{l}{GI-$z/w$ (GIAS, ours)} & \multicolumn{1}{l}{GI-$x$+BN     \cite{nvidia}} & \multicolumn{1}{l}{GI-$z/w$+BN (ours)} \\
        
        \midrule

        PSNR $\uparrow$    & $13.78$ & $14.27$ & $14.70$ & $\mathbf{15.58}$ & $15.52$ & $\mathbf{16.31}$ \\ 
        SSIM $\uparrow$    & $0.2542$ & $0.3106$ & $0.3519$ & $\mathbf{0.3895}$ & $0.3513$ & $\mathbf{0.4311}$ \\ 
        LPIPS $\downarrow$ & $0.4376$ & $0.3233$ & $0.5121$ & $\mathbf{0.3023}$ & $0.3645$ & $\mathbf{0.2861}$ \\ 
        
        \bottomrule
    \end{tabular}
}

\vskip -0.3cm

\end{table*}
    
        % \begin{figure}[t]
        %   \centering
        %   \includegraphics[width=\textwidth]{figures/expr_defense.pdf}
        %   \caption{
        %         Comparison of reconstruction methods with varying difficulties.
        %         GI-$z/w$ always surpasses GI-$x$ thanks to the pretrained generative model.
        %         All subfigures share the y-axis. 
                
        %         % Also shows less performance decreases with problem difficulties.
        %     %   (a) Gradient inversion performances varying batch size $B = 1, 2, 4, 8$, where 
        %     %     larger $B$ increases the reconstruction difficulty.
        %     %   (b) Gradient inversion performances 
        %     %   when the gradient is compressed 
        %     %   with different sparsity %(space saving) 
        %     %   from $0$\% to $99$\%. 
        %         % When compression rate gets high, our methods shows stable results.
        %         % Note that for ResNet-18, there are still enough parameters with compression rate 0.99.
        %   }
        % \end{figure}
        
        \begin{figure}[t]
            \begin{subfigure}[b]{0.326\textwidth}
              \includegraphics[height=4.4cm]{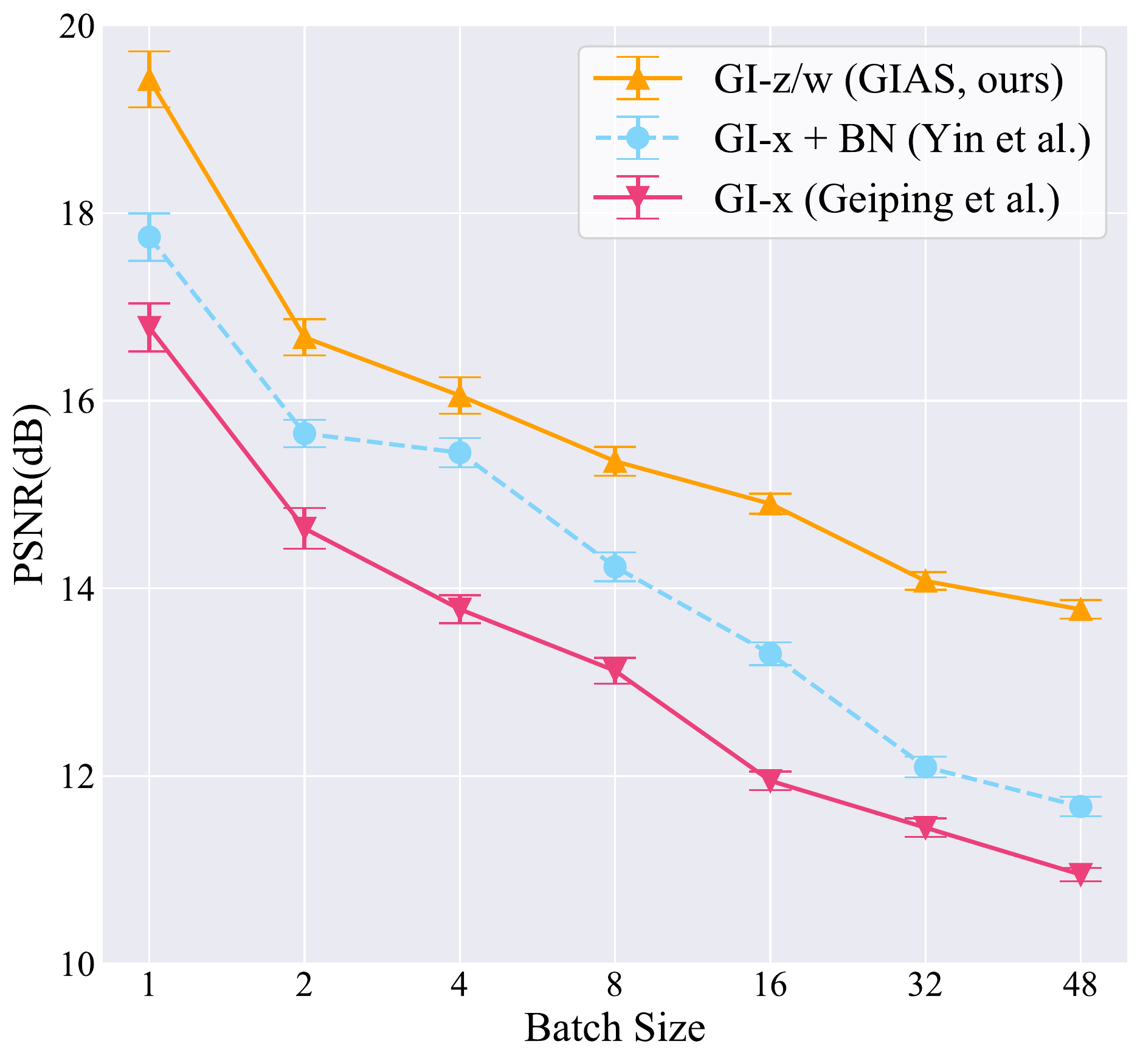} 
              \caption{}
              \end{subfigure}
            \begin{subfigure}[b]{0.3\textwidth}
              \includegraphics[height=4.35cm]{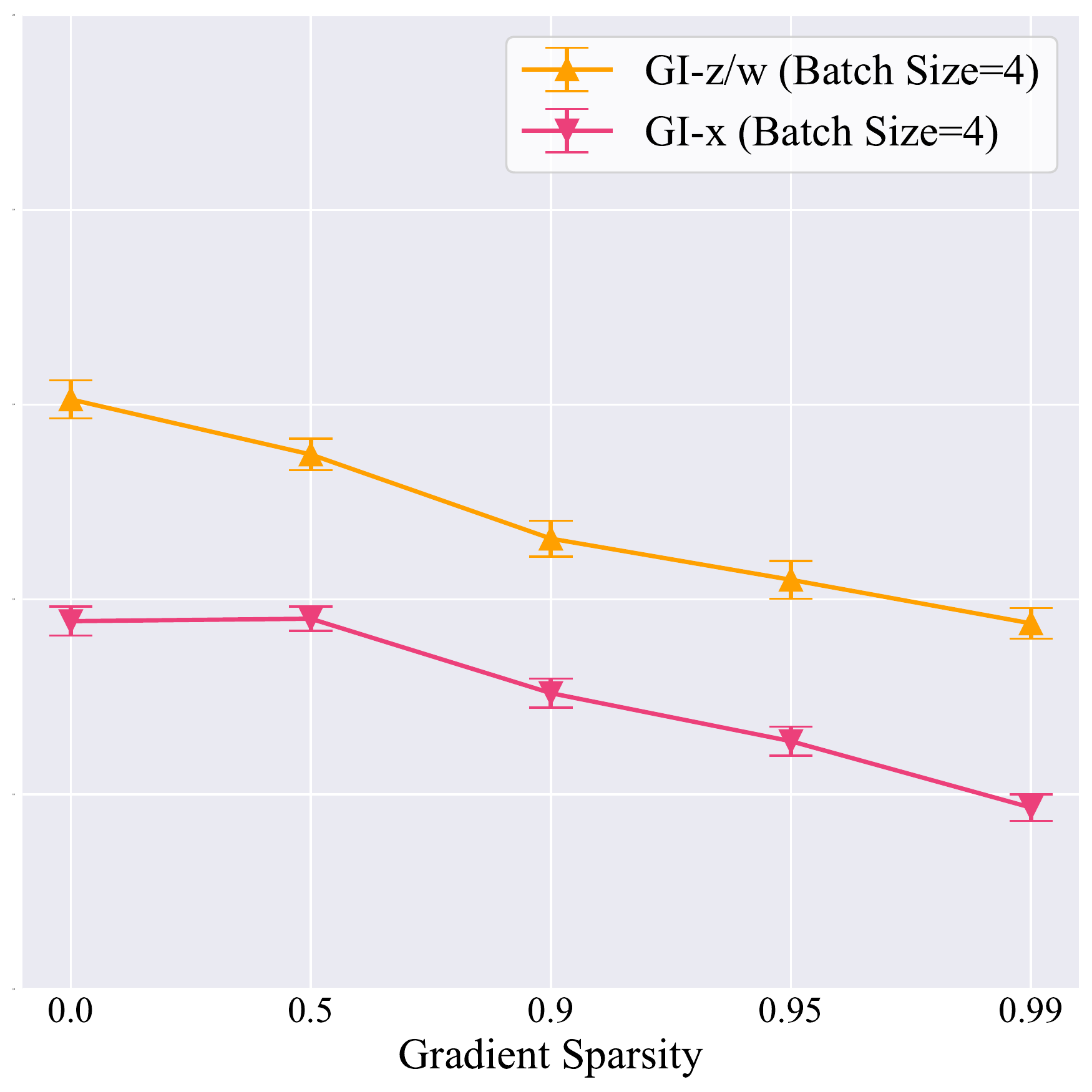} 
              \caption{}
            \end{subfigure}
            \begin{subfigure}[b]{0.3\textwidth}
              \includegraphics[height=4.35cm]{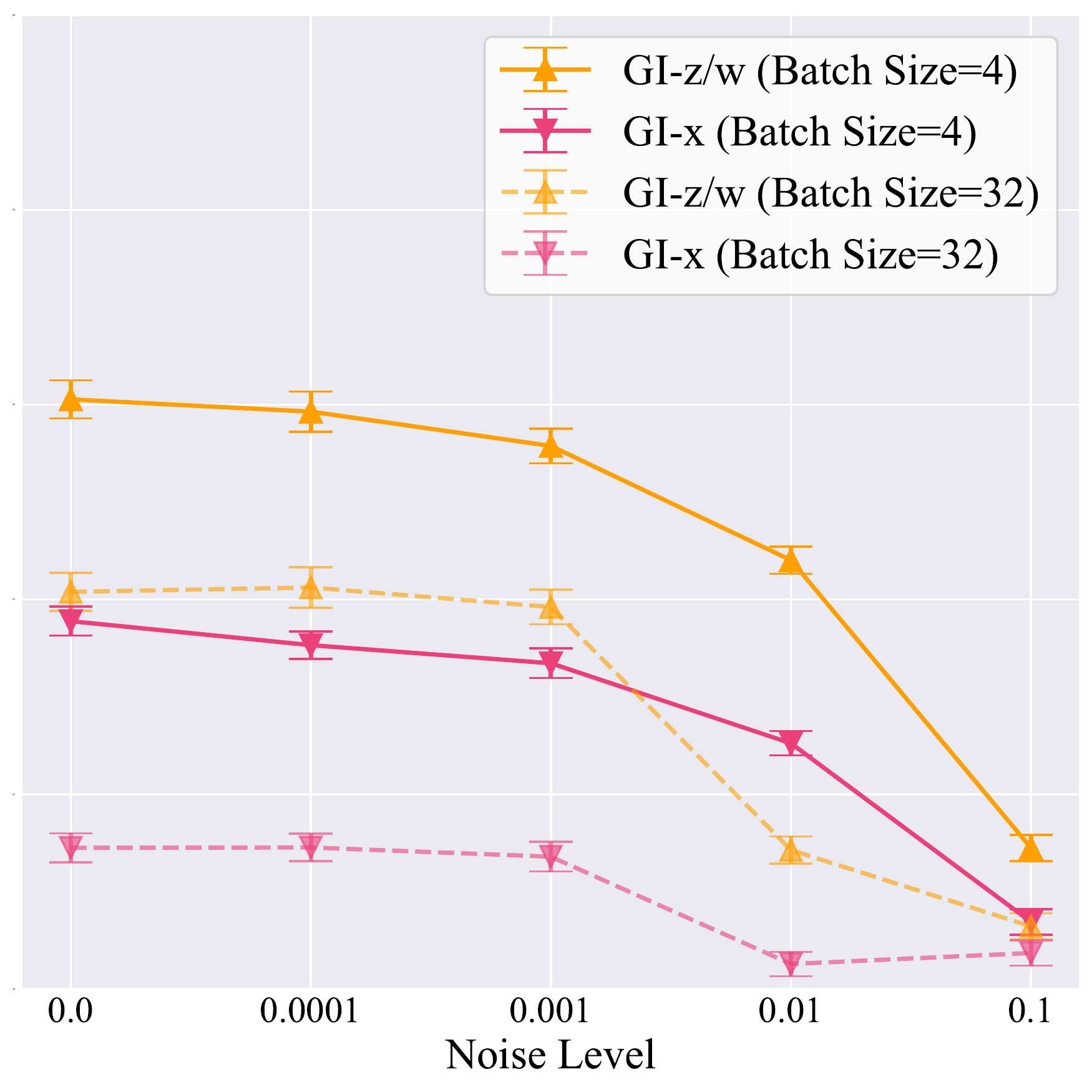} 
              \caption{}
            \end{subfigure}
            \caption{
            \tblue{
                {\em Comparison of \sota models and GI-$z/w$ with varying difficulties}.
                Larger batch size, higher sparsity, and larger gradient noise 
                increases reconstruction difficulty.
                GI-$z/w$ always surpasses GI-$x$ thanks to the pretrained generative model.
                All subfigures share the y-axis.
                    % Also shows less performance decreases with problem difficulties.
                %   (a) Gradient inversion performances varying batch size $B = 1, 2, 4, 8$, where 
                %     larger $B$ increases the reconstruction difficulty.
                %   (b) Gradient inversion performances 
                %   when the gradient is compressed 
                %   with different sparsity %(space saving) 
                %   from $0$\% to $99$\%. 
                    % When compression rate gets high, our methods shows stable results.
                    % Note that for ResNet-18, there are still enough parameters with compression rate 0.99.
            }}
            \label{figure:compressed}
        \end{figure}

    \subsection{The gain from fully exploiting pretrained generative model}

    \paragraph{Comparison with state-of-the-art models.} 
        Our method can be easily added to previous methods \cite{geiping, nvidia}.
        In Table~\ref{tab:against_sota}
        and Figure~\ref{figure:ablation}, 
        we compare the state-of-the-art methods 
        both with and without the proposed generative modelling.
        % Note that our baseline implementation for \cite{nvidia} includes fidelity regularizer with BN$_\text{exact}$ and group lazy regularizer, not group registration regularizer. 
        % Our baseline implementation might be imperfect, but it still shows adding our methods improves performance.
        In Table~\ref{tab:against_sota},
        comparing GI-$x$ to GI-$z/w$
        and GI-$x$ + BN to GI-$z/w$ + BN,
        adding the proposed generative modelling 
        provides additional  gain
        in terms of all the measures (PSNR, SSIM, LPIPS) of 
        reconstruction quality.
        GI-$z/w$ without BN has lower reconstruction error 
        than GI-$x$ + BN, which is the method of \cite{nvidia}.
        This implies that the gain from the generative model 
        is comparable to that from BN statistics. However,  while the generative model only requires a global (and hence coarse) 
        knowledge on
        the {\it entire dataset}, BN statistics are local to the batch in hand and hence requires significantly more 
        detailed information on the {\it exact batch} used to compute gradient.
        % In Table~\ref{tab:against_sota}, 
        % only the  latent space search and parameter space search shows better result.
        % { \color{blue}
        %We also note that     
        As shown in Figure~\ref{figure:ablation}, 
        the superiority of our method compared to the others
        is clear in terms of 
        the best-in-batch performance than the average one,
        where the former is more suitable to show actual privacy threat in the worst case than the latter.
        % }
        It is also interesting to note that GI-$w$ with untrained $w$
        provides  substantial gain compared to GI-$x$.
        This may imply that there is a gain of the implicit bias, c.f., 
        \cite{dip}, from training the architecture of deep generative model.

    \tblue{
    \paragraph{Evaluation against possible defense methods} 
        We evaluate the gain of using a generative model
        for various FL scenarios with varying levels of difficulty in the inversion. 
        As batch size, gradient sparsity\footnote{Having gradient sparsity 0.99\% implies that we reconstruct data
        from 1\% of the gradient after removing 99\% elements with the smallest magnitudes at each layer.}~\cite{wei} 
        and gradient noise level increase,
        the risk of having under-determined inversion
        increases and the inversion task becomes more challenging. 
        Figure~\ref{figure:compressed} shows that for all the levels of difficulty, the generative model provides
        significant gain in reconstruction quality. 
        % In particular, the quality of GI-$x$ \cite{geiping} with a batch size of four
        % (resp. gradient sparsity 0\% and noise level 0) is comparable 
        % to that of GIAS with a batch size of 32 (resp. gradient sparsity 99\% and noise level $0.01$).
        In particular, the averaged PSNR of GI-$x$ with a batch size of 4 is comparable
        to that of GI-$z/w$ with a batch size 32.
        It is also comparable to that of GI-$z/w$ with a gradient sparsity of 99\%. 
        To measure the impact of the noisy gradient, 
        we experimented gradient inversion with varying gaussian noise level in aforementioned settings.
        Figure~\ref{figure:compressed}(c) shows that adding enough noise to the gradient can mitigate the privacy leakage.
        GI-$z/w$ with a noise level of $0.01$, which is relatively large, still surpasses GI-$x$ without noise.
        A large noise of $0.1$ can diminish the gain of exploiting a pretrained generative model.
        However, the fact that adding large noise to the gradient slows down training 
        makes it difficult for FL practitioners to choose suitable hyperparameters.
        % Figure~\ref{figure:compressed}(c) shows the experiment results with varying gradient noise level.
        The results imply our method is more robust to defense methods against gradient inversion, but can be blocked by a high threshold.
        Note that our results of gradient sparsity and gradient noise implies the Differential Privacy(DP) is still a valid defense method, when applied with a more conservative threshold.
        For more discussion about possible defense methods in FL framework, see Appendix~\ref{sec:defense}.
        
    }

    \begin{figure}[t]
          \centering
          
      \begin{subfigure}[b]{0.55\textwidth}
        \includegraphics[width=\textwidth, trim={1cm, 0cm, 0cm, 2cm}]{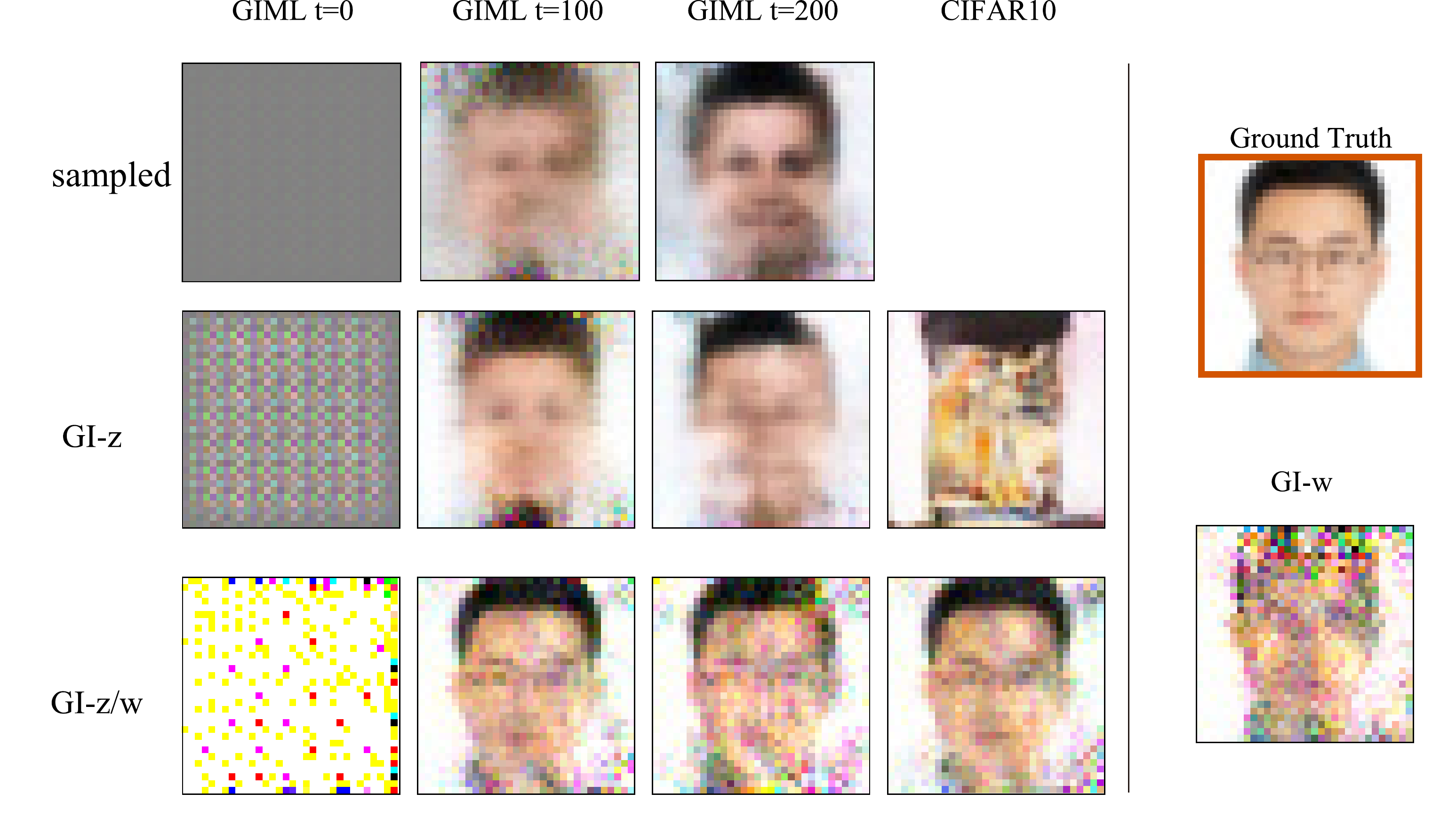}
        \caption{}
      \end{subfigure}
      \begin{subfigure}[b]{0.43\textwidth}
        \includegraphics[width=\textwidth, trim={0cm, 0cm, 0cm, 0cm}]{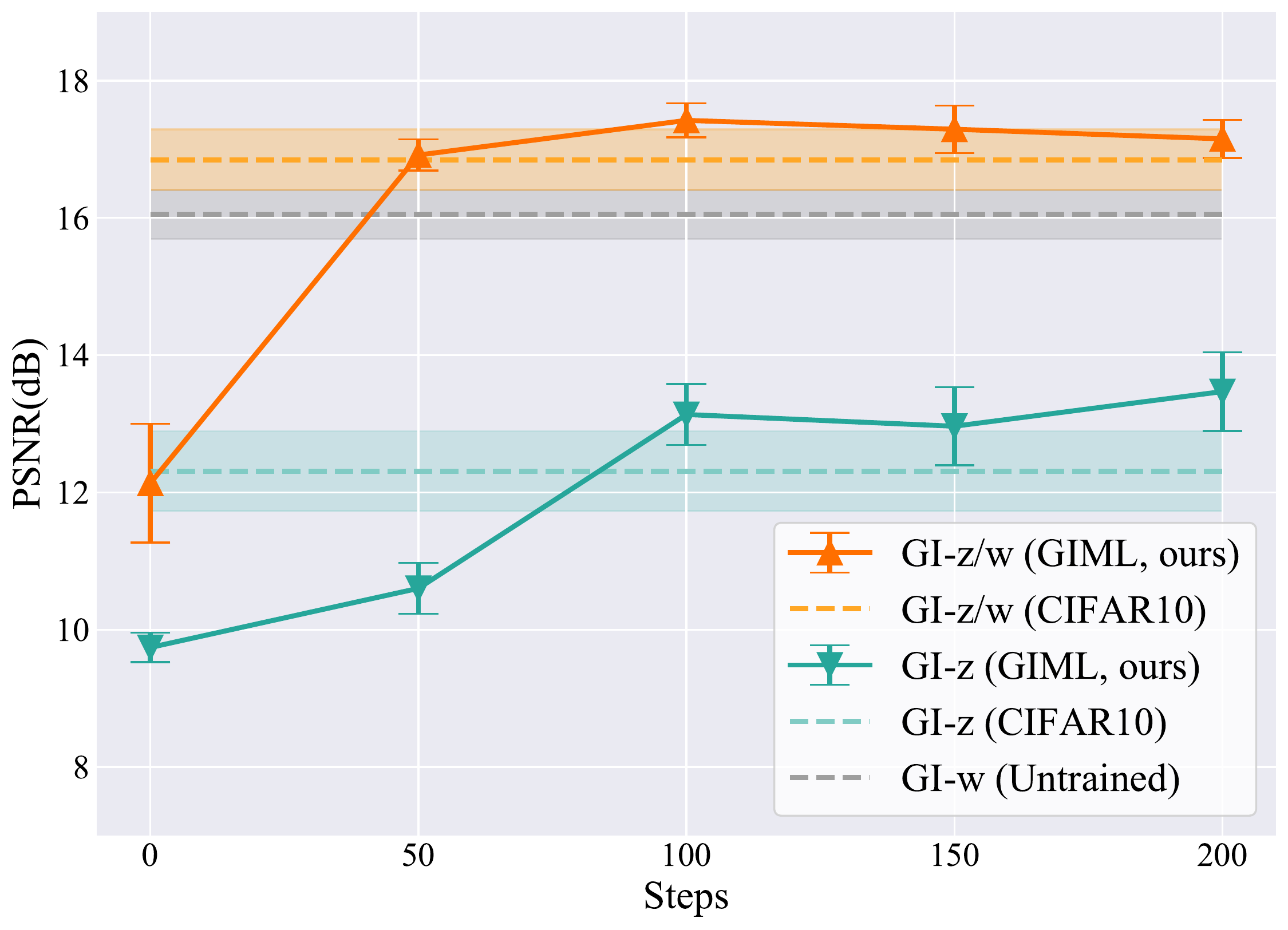}
        \caption{}
      \end{subfigure}
          \caption{
          \tblue{
                {\em Qualitative and quantitative result of GIML}.
              (a) 
                Results validating generative model trained with GIML.
                  Images on the first row are sampled from different GIML training steps. 
                %   Initial generative model is untrained DCGAN, and meta-learned with gradients computed on FFHQ images. 
                The same latent code $z$ was used to sample images in same rows.
                Images on the second row and third row are results of GI-$z$ and GI-$z/w$ using generative model trained with GIML and pretrained model which is trained with CIFAR10 images.
                Experiments were done with gradient sparsity 0.95 for comparison in difficult setting.
                Last column represents the ground truth image and result of GI-$w$ with untrained model.
                %   GI-$z/w$, GI-$z$ with CIFAR10 pretrained model and GI-$w$ with untrained model is in comparison.
                % Gain of using generative model trained with GIML.
                % The results shows the necessity of parameter space search and a generative model is trained from gradients.
                (b)
                    A comparison of GIAS with meta-learned generative model and GIAS using improper generative model. 
                    Proper generative model boosts GIAS performance. 
          }
          }
            \label{fig:trained_sampled}
            \label{fig:trained_result}
            \label{figure:sampled_trained}    
            \label{fig:trained_psnr}
    \end{figure}

\subsection{Learning generative model from gradients}
    We demonstrate the possibility of {\em training}
    a generative model only with gradients.
    For computational tractability, we use DCGAN and images from FFHQ~\cite{Karras_2019_CVPR} resized to 32x32.
    We generate a set of gradients
    from 4 rounds of gradient reports from 200 nodes,
    in which node computes gradient for a classification task based on the annotation provided in \cite{ffhqfeatures}. From the set of gradients, we perform GIML to train a DCGAN to potentially generate FFHQ data.
    
    \tblue{
    Figure~\ref{figure:sampled_trained} shows the evolution of generative model 
    improves the reconstruction quality when performing either GI-$z$ and GI-$z/w$.
    We can clearly see the necessity of parameter space search. 
    Figure~\ref{figure:sampled_trained}(a) shows that 
    the quality of images from the generative model 
    is evolving in the training process of GIML.
    As the step $t$ of GIML increases, the generative model $G_{w^{(t)}}(z)$ for arbitrary $z$ outputs more plausible image of human face.
    When using generative model trained on wrong dataset (CIFAR10),
    %different than FFHQ, 
    GI-$z$ completely fails at recovering data.
    
    % in which we empirically demonstrate that the latent space search using CIFAR10 pretrained model does not generate face-looking images.
    % Other methods were able to find an image close to the target image,  
    In Figure~\ref{figure:sampled_trained}(b), as GIML iteration step increases, the performance of GI-$z$ and GI-$z/w$ with GIML surpass GI-$z$ and GI-$z/w$ with wrong prior knowledge.
    GI-$z/w$ using generative model trained on wrong dataset and 
    GI-$w$ which starts with an untrained generative model
    show lower averaged PSNR 
    compared to GI-$z/w$ with GIML.
     GI-$z/w$
     with GIML to train generative model on right data shows the best performance
     in terms of not only quality
     (Figure~\ref{figure:sampled_trained})
     but also convergence speed.
     We provide a comparison of the convergence speed in Appendix~\ref{sec:convergence_speed}.
    
    % In Figure \ref{figure:sampled_trained}, performance of parameter space search with untrained model,
    % known as Deep Image Prior, is too powerful so effect of generative model is ignored.
    % To check effect of using generative model more clearly, 
    % we created noisy image dataset by random-shuffling imagenet data.
    % Random shuffled data are known to be difficult to find through Deep Image Prior.
    % With meta-learned generative model, performance of our algorithm works better than DIP with untrained DCGAN.
    }

% Conclusion
\section{Conclusion}
\label{sec:conclusion}

We propose GIAS fully exploit 
the prior information on user data 
from a pretrained generative model
when inverting gradient. 
We demonstrate significant privacy leakage  using GIAS with pretrained generative model
in various challenging scenarios, where our method provides 
substantial gain 
additionally
 to  any other existing methods \cite{geiping, nvidia}.
In addition, we propose GIML which can train a generative model using only the gradients seen  in the  FL classifier training. 
We experimentally show that GIML can meta-learn a generative model
on the user data from only gradients,
which improves the quality of each individual  recovered image.
% provides the same level of data recovery with a given pre-trained model.
To our best knowledge, GIML is the first capable of learning explicit 
prior on a set of gradient inversion tasks.

\section*{Acknowledgments}

This work was partly supported by Institute of Information \& communications Technology Planning \& Evaluation (IITP) grant funded by the Korea government (MSIT) (No. 2019-0-01906, Artificial Intelligence Graduate School Program (POSTECH)) and (No. 2021-0-00739, Development of Distributed/Cooperative AI based 5G+ Network Data Analytics Functions and Control Technology).
Jinwoo Jeon and Jaechang Kim were supported by the Institute of Information \& Communications Technology Planning \& Evaluation (IITP) grant funded by Korea(MSIT) (2020-0-01594, PSAI industry-academic joint research and education program).
Kangwook Lee was supported by NSF/Intel Partnership on Machine Learning for Wireless Networking Program under Grant No. CNS-2003129 and NSF Award DMS-2023239.
Sewoong Oh acknowledges funding from NSF IIS-1929955, NSF CCF 2019844, and Google faculty research award.

\bibliography{ref}

\clearpage

\appendix

\section*{Appendix}

% change figure number format to A1, A2, ... 
\renewcommand{\thefigure}{A\arabic{figure}}
\setcounter{figure}{0}

      \section{Detailed algorithms}
      \label{sec:algorithm}

% \begin{figure*}[H]
\begin{algorithm}[!ht]
    \newcommand{\batch}{\mathcal{B}}
    \caption{Gradient Inversion in Alternative Spaces (GIAS)}
    \label{algo:GIAS}
    \begin{algorithmic}[1]
        \REQUIRE
        learning model $f_\theta$;
        target gradient $g = \nabla_\theta \sum_{j=1}^B \ell(f_\theta(x_j^*), y_j^*)$ 
        to be inverted;     
        batch size B; pre-trained generative model $G_w$; %number of local steps $\tau_z$, $\tau_w$; learning rates $\eta_z, \eta_w$
            %; total variation coefficient $\lambda_{\text{TV}}$
        
        \smallskip
        
        \STATE Initialize %latent vectors 
        $\boldsymbol{z} := ( z_{1}, ..., z_{B})$ 
        randomly
        %from a distribution (e.g., zero-mean Gaussian)
        
        \STATE Find  $\boldsymbol{z} \leftarrow \argmin_{\boldsymbol{z}} 
        c( G_{w} (z_1), ..., G_{w} (z_B))$
        \\
        \COMMENT{Latent space search}%~\footnotemark{\label{ft:joint}}}

        \STATE 
        Set
        $\boldsymbol{w} := (w_1, \dots, w_B) \leftarrow (w, \dots, w)$
        %instance-wisely
        %\COMMENT{  generative models}
        
        \STATE Find $\boldsymbol{w} \leftarrow \argmin_{\boldsymbol{w}} 
        c(G_{w_1} (z_1), \dots, G_{w_B} (z_B))$
        \\
                \COMMENT{Parameter space search} %\footnoteref{ft:joint}}

            % $ w' \leftarrow 
            %     \mathop{\mathrm{argmin}}_{w_1, \dots, w_B} 
            %     d \left( G_{w_1} (z_1), \dots, G_{w_B} (z_B) ; \theta, \batch_{i})
            %                 \right) $ 
        
        \STATE Return result:  $ G_{w'_1}(z_1), \dots, G_{w'_B}(z_B) $
        
    \end{algorithmic}   
\end{algorithm}
% \end{figure*}

    \begin{algorithm}[!ht]
        \newcommand{\data}{\mathcal{D}}
        \newcommand{\batch}{\mathcal{B}}
        \newcommand{\loss}{\mathcal{L}}
        \newcommand{\inp}{\mathbf{x}}
        \newcommand{\learner}{f_\theta}
        \newcommand{\lossi}{\loss_{\batch_i}}
    
        \caption{Gradient Inversion to Meta-Learn generative model (GIML)}
        \label{algo:gen-meta}
        \begin{algorithmic}[1]
            \REQUIRE 
            inversion task set $\mathcal{S}$;
            task batch size $N$; data batch size $B$ (per gradient); number of local iterations $\tau$; $z$-regularizer coefficient $\lambda$;
step sizes $\alpha$, $\beta$;
             %notation? S?
            \STATE Initialize $w$ randomly
            \WHILE{not done}
                \STATE Sample
                a batch of inversion tasks $(\theta_1, g_1), ..., (\theta_N, g_N)$ from $\set{S}$
                %for each node $i = 1,..., N$ 
                \STATE $ w' \leftarrow w$ 
                \FORALL{$i = 1,\dots, N$} 
                     \STATE $\vec{z}^*_i 
                     %(z^*_{i1}, \dots, z^*_{iB})
                     \leftarrow \mathop{\mathrm{argmin}}_{\vec{z}_i} 
                            c ( G_{w'} (z_{i1}), \dots , G_{w'} (z_{iB}); \theta_i, g_i ) 
                            + \lambda \sum_j  \| z_{ij} \|_2 $ 
                            \COMMENT{Regularized latent space search}
                            % \nabla_\theta\ell ( f_\theta( G_{w'} ( (z_{i1}, \dots, z_{iB}))), g_i) 
                            % \right) + \lambda \sum_j  \| z_{ij} \|_2  $ 
                    %  \STATE 
                    %     $(z_{i1}, \dots, z_{iB}) \leftarrow \mathop{\mathrm{argmin}}_{(z_{i1}, \dots, z_{iB})}  d \left( 
                    %         \nabla_\theta\ell ( f_\theta( G_{w'} ( (z_{i1}, \dots, z_{iB}))), g_i) 
                    %         \right) + \lambda \sum_j  \| z_{ij} \|_2  $ 
                 \ENDFOR
                     
                     \FORALL{$t = 1, \dots, \tau$}
                        %  \STATE  % Evaluate 
                        %     $ \loss =  \sum_i 
                        %     c \left( 
                        %     G_{w'} (z^*_{i1}), \dots, G_{w'} (z^*_{iB}); \theta_i, g_i
                        %     \right) $
                            % d \left( 
                            % \nabla_\theta\ell ( f_\theta( G_{w'} (z_{i1}, \dots, z_{iB})), g_i) 
                            % \right) $
                         \STATE 
                         %\resizebox{.85\hsize}{!}{
                         $w' \leftarrow w' -\alpha \nabla_{w'} 
                         \sum_i 
                            c \left( 
                            G_{w'} (z^*_{i1}), \dots, G_{w'} (z^*_{iB}); \theta_i, g_i
                            \right)$
                            %}
                         \COMMENT{Meta parameter space search}
                     \ENDFOR
                 \STATE Update
                 $w \leftarrow w -
                 \beta (w-w') = (1-\beta) w
                 + \beta w'$
            \ENDWHILE
        \end{algorithmic}
    \end{algorithm}

\section{Proof of Property~\ref{prop:conti}}
\label{sec:property}

To prove Property~\ref{prop:conti},
we first conclude the same statement of Property~\ref{prop:conti} 
assuming the inversion problem is continuous at $x^*$ (Lemma~\ref{lem:conti2}),
and then show that the standard scenario described in Property~\ref{prop:conti}
guarantees the desired continuity  (Lemma~\ref{lem:conti-prob}).
The canonical form of learning model mentioned in \ref{prop:conti}
is described by
the assumptions of Lemma~\ref{lem:conti-prob}.
%in the assumptions of Lemmas~~\ref{lem:conti2}~and~\ref{lem:conti-prob}.

\begin{lemma}[An extension of Property~\ref{prop:conti} in the main text] \label{lem:conti2}
      %Suppose that
      For an input data $x^* \in {[0,1]^m}$, 
      consider the gradient inversion problem of minimizing cost $c(x)$
      in \eqref{eq:GIP-sim} where $c(x)$ is continuous.
%       , i.e., 
%     for given any $\epsilon >0$, there exists $\delta >0$ such that if $\|x - x^*\| < \delta$ then
% $|c(x; \theta, g^*)- c(x^*; \theta, g^*)|< \epsilon$ where $g^*$ is the gradient computed using $f_\theta$ and $x^*$. 
    %   the learning model $f_\theta$ is 
    %   a standard form of neural network 
    %   (such as multi-layer perceptron or convolutional neural network)
    %   with C1 ({continuously differentiable}) activations
    %   (e.g., sigmoid and exponential linear), 
    %   loss function $\ell$ is C1  (e.g., logistic and exponential),
    %   and the discrepancy measure $d$ is $\ell_2$-distance. 
%       $\ell_2$-distance
    Suppose that it has the unique global minimizer at $x^*$.
      Let $\varepsilon \ge 0$ 
      be the approximation error bound
      on $x^*$ for generative model $G_w: {[0,1]^k} \mapsto {[0,1]^m}$
      with $k \le m$, i.e.,
      $\min_{z \in{[0,1]^k}} \|x^* -  G_w(z)\| \le  \varepsilon$.
    Then, there exists $\delta(\varepsilon) \ge 0$ such that for any $z^* \in \argmin_{z} c(G_w(z))$, % verifies
        \begin{align}
        \| G_w(z^*) - x^* \| \le \delta(\varepsilon) \;,
        \end{align}
        of which upper bound $\delta(\varepsilon) \to 0$ as $\varepsilon \to 0$.
      \end{lemma}
      \begin{proof}[Proof of Lemma~\ref{lem:conti2}]
      From the assumptions that $x^*$ is the unique minimizer
      and $c(x)$ is continuous on $[0,1]^m$,
      it follows that for $x \in  [0,1]^m$, 
      if $c(x) \to c(x^*)$, then $x \to x^*$. 
      This can be proved by contradiction.
    %   Suppose not, i.e.,
    %   $c(x) \to c(x^*)$ and $x \not\to x^*$. Then, the supposition creates 
    %   contradiction to at least one of the assumptions.      
    %   then 
    %   $c(x)$
      Then we have that 
      for $\varepsilon > 0$, there exists $\delta(\varepsilon) > 0$
      such that if $c(x) \le \varepsilon$, then $\|x - x^*\| \le \delta(\varepsilon)$ where $\delta(\varepsilon) \to 0$ as $\varepsilon \to 0$.
      From the continuity of $c(x)$, it is straightforward to check that
      $c(G_w(z^*(\varepsilon))) \to c(x^*)$ as $\varepsilon \to 0$.
      This completes the proof.
%       $G_w$
%       for any given $\varepsilon \ge 0$,
%       there
%       Let $\set{X}^*:=\{x' \in [0,1]^m: \exists \delta>0 ~\text{s.t.}~ c(x) \ge c(x')
%       ~\forall x \in [0,1]^m 
%       ~\text{and}~ \|x- x'\| \le \delta \}$ 
%       be the set of local minimizers, and let $\varepsilon^* := \min_{x \in \set{X}^* \setminus x^*} c(x) - c(x^*) >0$
%       be the gap between the unique global minimum $c(x^*)$ and the others. 
%       Thanks to the continuity assumption on $c(x)$, there exists
%       We assume that $x^*$ is the unique global minimizer of $c(x^*)$.
% This implies that there exist $\delta^* > 0$ and $\varepsilon^* > 0$ 
% such that if $\|x - x^*\| \ge \varepsilon^*$, then $c(x) + \delta^* \ge c(x^*)$.
% Let $z' \in \argmin_z \|G_w (z) - x^*\|$ be the latent code
% for $x^*$. Then, from the assumption on the accuracy of the generative model, it follows that $\|G_w(z') - x^*\| \le \varepsilon$. 
% Let $\varepsilon \ge 0$ be the approximation
% Let $z^* \in \argmin_z c(G_w (z))$ from the gradient inversion.
% Then, $c(G_w (z^*)) \ge c(G_w (z')) \ge 0$.
      \end{proof}

\begin{lemma} \label{lem:conti-prob}
      %Suppose that
      For an input data $x^* \in \mathbb{R}^m$, 
      consider
      the gradient inversion problem of minimizing cost $c$
      in \eqref{eq:GIP-sim},
      where
      %learning model
      the learning model $f_\theta$ is 
      a standard form of $R$-layer neural network $f_\theta(x)= \Theta_R \sigma_{R-1}(\Theta_{R-1}\sigma_{R-2}(... \Theta_1 x))$
      with $\Theta_{r} \in \mathbb{R}^{m_r \times m_{r-1}}$ for each $r$, where $m_R= L$
      and $m_0 = m$,
      and  C1 ({continuously differentiable}) activation $\sigma$'s
      (e.g., sigmoid and exponential linear), 
      loss function $\ell$ is C1  (e.g., logistic and exponential),
      and the discrepancy measure $d$ is $\ell_2$-distance. 
%       $\ell_2$-distance
    Then, the corresponding cost function $c(x)$
    is continuous with respect to $x \in \mathbb{R}^m$.
      \end{lemma}
      \begin{proof}[Proof of Lemma~\ref{lem:conti-prob}]
      Note that the standard model $f_\theta$ includes multi-layer perceptron or convolutional neural network.
      It is $C1$ since the composition of C$k$\footnote{the $k$-th derivative is continuous} functions is C$k$.
      %and the standard form of neural network is continuous. Hence, directly from the assumption,
      Hence, the gradient is continuous w.r.t. $x$.
      In addition, the cost function $c(x)$
      is continuous since the gradient and the choice of discrepancy measure are continuous.
      This concludes the proof.
      \end{proof}
      
      The proof of Property~\ref{prop:conti} is straightforward from
      Lemmas~\ref{lem:conti2}~and~\ref{lem:conti-prob}.
      
%       \begin{property} \label{prop:conti2}
%       %Suppose that
%       For an input data $x^* \in \mathbb{R}^m$, 
%       consider
%       the gradient inversion problem of minimizing cost $c$
%       in \eqref{eq:GIP-sim},
%       where
%       %learning model
%       the learning model $f_\theta$ is 
%       a standard form of neural network 
%       (such as multi-layer perceptron or convolutional neural network)
%       with C1 ({continuously differentiable}) activations
%       (e.g., sigmoid and exponential linear), 
%       loss function $\ell$ is C1  (e.g., logistic and exponential),
%       and the discrepancy measure $d$ is $\ell_2$-distance. 
% %       $\ell_2$-distance
%     Suppose that it has the unique global minimizer at $x^*$.
%       Let $\varepsilon \ge 0$ 
%       be the approximation error bound
%       on $x^*$ for generative model $G_w: \mathbb{R}^k \mapsto \mathbb{R}^m$
%       with $k \le m$, i.e.,
%       $\min_{z \in \mathbb{R}^k} \|x^* -  G_w(z)\| \le  \varepsilon$.
%     Then, there exists $\delta(\varepsilon) \ge 0$ such that for any $z^* \in \argmin_{z} c(G_w(z))$, % verifies
%         \begin{align}
%         \| G_w(z^*) - x^* \| \le \delta(\varepsilon) \;,
%         \end{align}
%         of which upper bound $\delta(\varepsilon) \to 0$ as $\varepsilon \to 0$.
%       \end{property}

    \begin{figure}[ht]
    \vskip 0.2in
      \centering
      
      \includegraphics[width=\textwidth, trim={0.8cm, 0cm, 1cm, 0cm}]{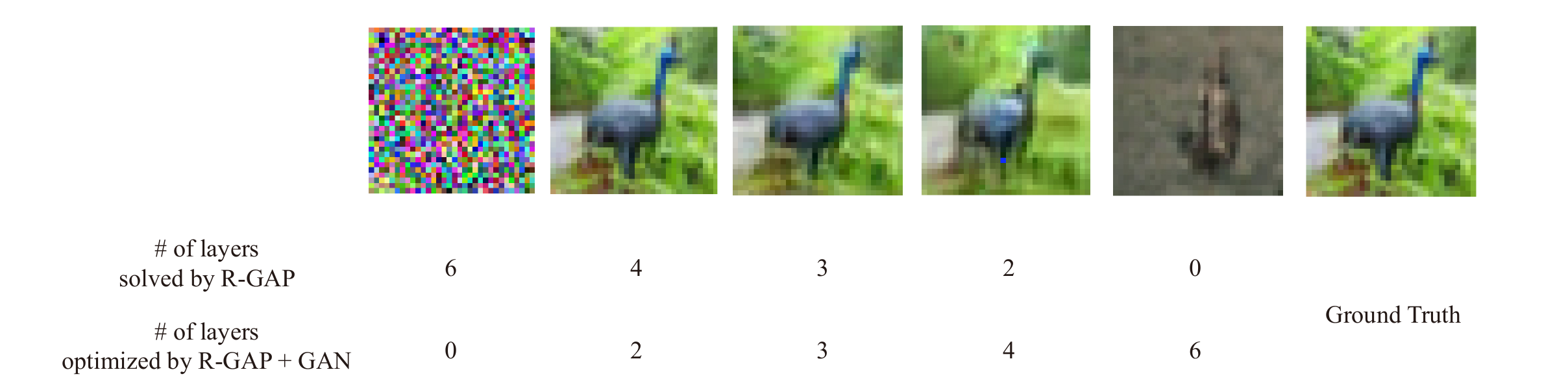}
      \caption{
          Comparison of R-GAP and R-GAP with a generative model.
          The second convolution layer is rank-deficient and R-GAP should solve under-determined problem.
          An under-determined problem is solved by using generative model.
          However, the error per layer increases much faster than R-GAP.
      }
      \label{figure:rgap}
    \vskip -0.2in
    \end{figure}

    \begin{figure}[ht]
    \vskip 0.2in
      \centering
      
      \includegraphics[width=0.6\columnwidth]{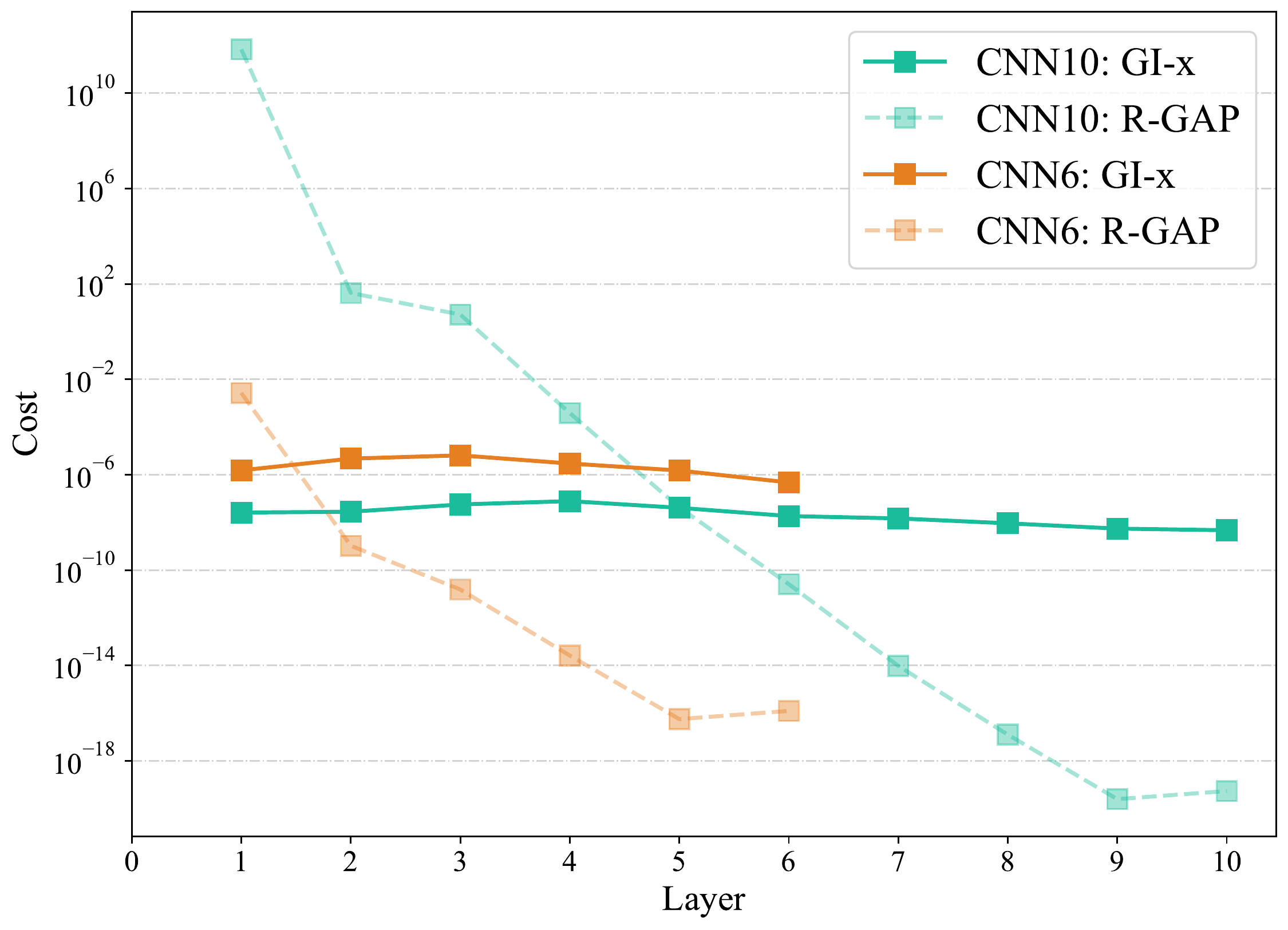}
      \caption{
          Layer-wise errors of convolution layers in gradient inversion attack results.
          layer-by-layer method shows cumulative exploding error.  
        %   \emph{CNN6: GI-$x$} and \emph{CNN10: GI-$x$} denote reconstruction method, which optimizes in $x$ space.
        %   \emph{CNN6: R-GAP} and \emph{CNN10: R-GAP} denote reconstruction method, which optimizes layers one-by-one.
        CNN6 denotes a neural network consists of six convolution layers and one FC layer.
        CNN10 denotes a neural network consists of ten convolution layers and one FC layer.
      }
      \label{figure:layer-loss}
    \vskip -0.2in
    \end{figure}

      \section{Another method for gradient inversion: R-GAP~\cite{rgap}}
        \label{sec:rgap}

In the main text, to solve the inversion problem in \eqref{eq:GIP}, 
we use gradient descent method directly to the cost function, while
we alternate the searching space. Meanwhile, \citet{rgap}
propose another approach, called R-GAP (recursive gradient attack on privacy), to solve the optimization in \eqref{eq:GIP}, although
it is limited to the case when %the discrepancy measure $d(\cdot, \cdot)$ is $\ell_2$-distance and 
the learning model $f_\theta$ is given as a standard form described in Lemma~\ref{lem:conti-prob}.
R-GAP decomposes the optimization \eqref{eq:GIP} into a sequence of
linear programming 
to reconstruct the output of each layer except the last layer's, and then it solves them recursively from the penultimate layer to the input layer. The linear programming to find
the output $x_r$ of the $r$-th layer can be written as follows:
\begin{align} \label{eq:rgap-lin}
A_r x_r = b_r
\end{align}
where $A_r$ and $b_r$ is a matrix and vector
depending on the previously reconstructed $x_{r+1}$, 
the parameter $\Theta_r$ of the $r$-the layer and its gradients.
For the definition of $A_r$ and $b_r$, we refer to \cite{rgap}.
Since each linear programming has a closed-form solution $A_r^\dagger b_r$,
this approach can be sometimes useful in terms of reducing computational cost. 

\paragraph{R-GAP with generative model.}
Note that the problem in \eqref{eq:rgap-lin}
can be rewritten as follows:
\begin{align} \label{eq:rgap-opt}
\min_{x_r} \|A_r x_r - b_r\| \;.
\end{align}
Let $f_{\theta, r}(x)$ be the output of the $r$-th layer.
Then, we can interpret $f_{\theta, r} (G_w(z))$ as a generative model for $x_r$.
Hence, the recursive reconstruction can be partially or fully
replaced with the following optimization:
\begin{align} \label{eq:rgap-opt-gen}
\min_{z, w} \|A_r f_{\theta, r}(G_w(z)) - b_r\| 
\end{align}
where the search space can be alternated arbitrarily.

\paragraph{A limitation of R-GAP.}

Such a use of generative model in \eqref{eq:rgap-opt-gen}
provides the same gain from reducing searching space. We however want to note that
it inherits the limitation of R-GAP, in which the reconstruction error in upper layers
propagates to that in lower layer. Hence, as the learning model becomes deeper, 
the reconstruction quality decreases while the number of parameters is increasing.
Figure~\ref{figure:layer-loss} shows the phenomenon of error accumulation of R-GAP.
It is possible that the optimization method in \eqref{eq:rgap-opt-gen}
can have lager error than the closed-form solution $A_r^\dagger b_r$
due to imperfection in generative model. Therefore, 
it is better to not use the generative model when the original linear programming is over-determined or determined. Indeed, in Figure~\ref{figure:rgap}, 
we present a trade-off between the linear programming in \eqref{eq:rgap-opt-gen}
and the optimization with generative model in \eqref{eq:rgap-opt-gen}, in which to emphasize
the trade-off, we perform the latent space search over $z$ only.
We obtain a substantial gain
from using the generative model for a few layers (one or two), whereas the gradient inversion is failed when using the generative model for every layer.

    % In previous section, we note the possibility of finding optimal solution with layer-by-layer methods decreases with respect to the number of layers.
    % With simple examples of 6-layer and 10-layer CNN based network, 
    % layer-by-layer method shows that the error explodes like Figure~\ref{figure:layer-loss}.
    % These result implies layer-by-layer methods could be applied to shallow network, 
    % but probably not for deep networks.

    \begin{figure*}[ht]
    \vskip 0.2in
      \centering
      
      \begin{subfigure}[b]{0.49\textwidth}
      \includegraphics[width=\textwidth]{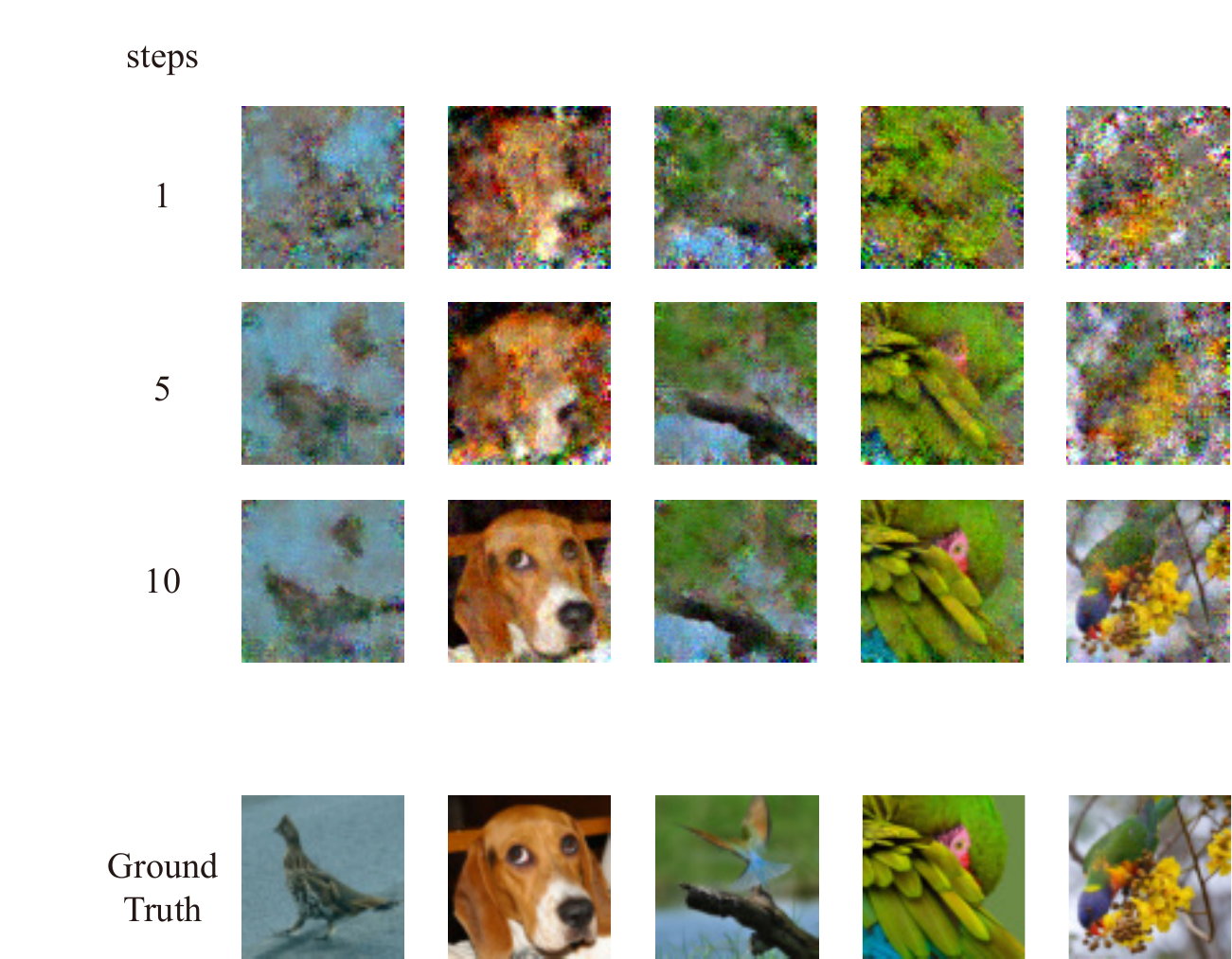}
      \caption{}
      \end{subfigure}
      \begin{subfigure}[b]{0.49\textwidth}
      \includegraphics[width=\textwidth]{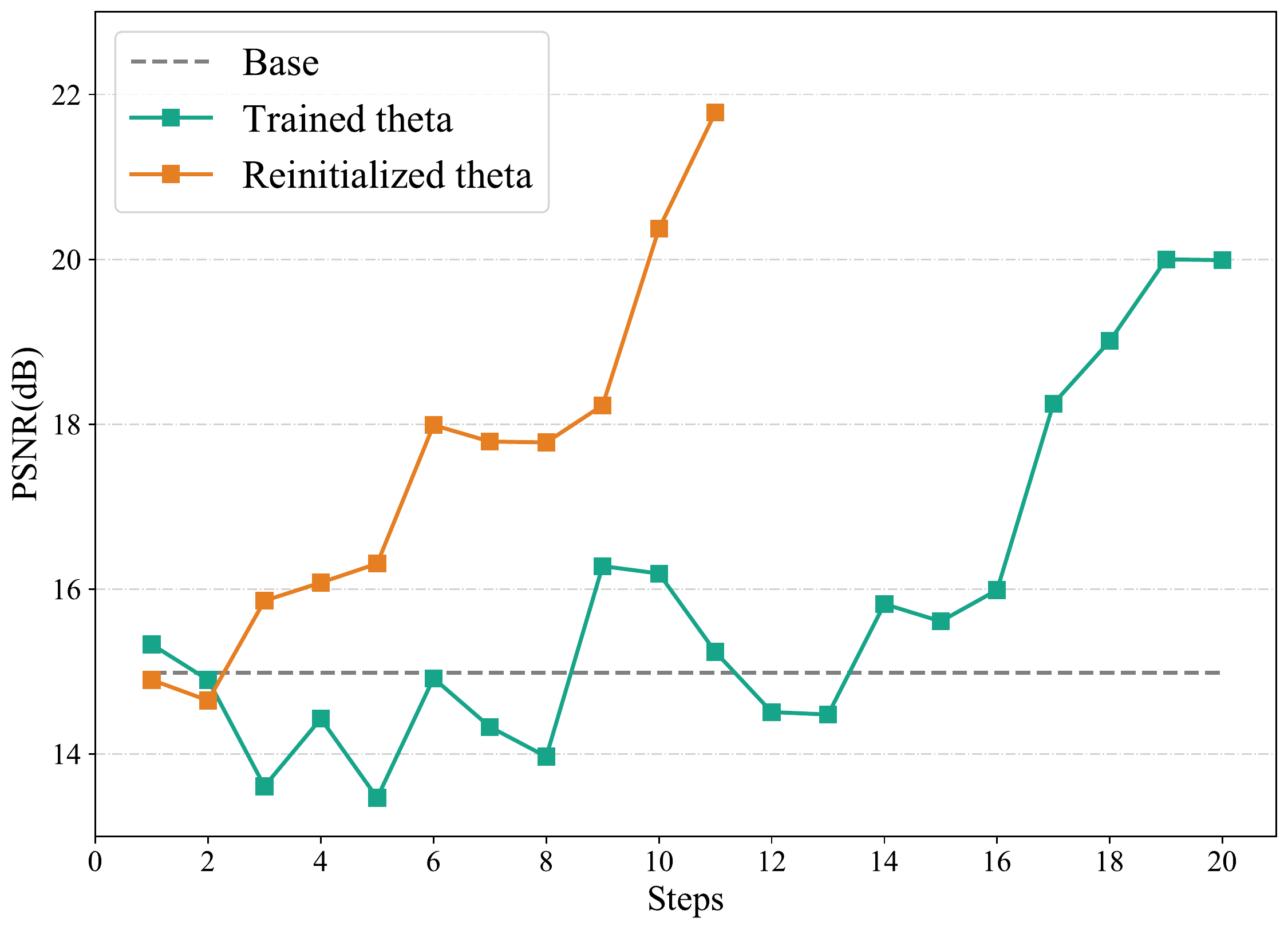}
      \caption{}
      \end{subfigure}
      \caption{
          (a)
          Examples of reconstructed images with a sequence of gradients.
          A local dataset contains eight images and generates gradients using four randomly chosen images in every step.
          As time $t$ increases, reconstructed images become more accurate.
        %   For images not contained in a batch, those images remains blank.
            (b)
            PSNR of reconstructed images. 
            In reinitialized theta setting, the classification model is reinitialized every step.
            In trained theta setting, the classification model is trained every step.
      }
      \label{figure:seq_examples}
    \vskip -0.2in
    \end{figure*}
    % \vspace{fill}

\section{Another potential gain from inverting a set of gradients}
\label{sec:sequence-of-gradients}
    In the main text, we demonstrate that from multiple gradients,
    we can train a generative model and use it to break the fundamental limit
    of inverting gradient solely.
    Beside this, assuming that we can observe a large number of gradients for the same data but different model parameters, it is able to reconstruct data almost perfectly
    by solving $\min_x \sum_{t = 1}^T c(x;\theta_t, g_t)$. 
    Such an assumption may be valid once we obtain the meta information
    to match gradients and data to be reconstructed.
    In Figure~\ref{figure:seq_examples}, we demonstrate this potential gain when 
    there are eight images only, but we observe a sequence of gradients obtained from the procedure of FL (the green curve). Of course, in the procedure of FL, the model parameter $\theta$
    slowly changes and thus the gain is smaller than that when each gradient is computed at
    completely random model parameters. However, in both settings, the reconstruction
    eventually becomes perfect as the observed gradients are accumulated.

\section{Strong generative prior}
\label{sec:ffhq_prior}
\label{sec:condition}

    When we have stronger prior on the data distribution, the gain from the generative model
    becomes larger. To show this, we use FFHQ \cite{Karras_2019_CVPR}
    rather than ImageNet in Section~\ref{sec:ablation}, where we believe FFHQ containing human-face images has less diversity than ImageNet including images of one thousand classes.
    % Generative model trained with FFHQ easily generate face images with proper face landmark locations.
    With FFHQ data, even GI-$z$ significantly outperforms GI-$x$, 
    while the gap between GI-$z$ and GI-$x$ is small for ImageNet in Figure~\ref{figure:ablation_image}.
    %GI-$z$ fails at capturing details. 
    This suggests a new approach to use a conditional generative model and data label $y^*$ in order for enjoying the gain from narrowing down the set of candidate input data by conditioning the label.
    %, although we use the label information to compute the gradient more accurately. 

    % latent space search easily finds skin color and view points, but having difficulty finding complicated features like gender.
    % In GAN Inversion, usage of perceptual loss alleviate this issue.
    % However, in gradient inversion problem, it is impossible because gradient is given, not image.
    % In the same reason, gradient inversion of trained classification model might be easy for our method.
    % As gradients generated from trained classification model already contain good features,
    % optimization of gradient inversion problem may consider deep features.

% \subsection{GIML with compressed gradients}

    % \begin{figure*}[t]
    % \vskip 0.2in
    %   \centering
      
    %   \begin{subfigure}[b]{0.49\textwidth}
    %   \includegraphics[width=\textwidth]{figures/exp_batch_size.pdf}
    %   \caption{}
    %   \label{figure:exp_batch_size}
    %   \end{subfigure}
    %   \begin{subfigure}[b]{0.49\textwidth}
    %   \includegraphics[width=\textwidth]{figures/exp_noise.pdf}
    %   \caption{}
    %   \label{figure:exp_noise}
    %   \end{subfigure}
    % %   \includegraphics[width=\textwidth]{local_dataset_v3.png}
    %   \caption{
    %       (a)
    %       Average PSNR of reconstructed images varying batch size up to $48$.
    %       (b)
    %       Average PSNR of reconstructed images varying noise level up to $0.1$.
    %   }
    %   \label{figure:exp_diff_privacy}
    % % \vskip -0.2in
    % \end{figure*}

\section{Possible defense methods against gradient inversion attacks}
\label{sec:defense}
In this section, we briefly discuss several defense algorithms against gradient inversion attacks.

\paragraph{Disguising label information.}
    As a defense mechanism specialized for the gradient inversion with generative model,
    we suggest to focus on mechanisms confusing the label reconstruction.
    In Section~\ref{sec:ablation}~and~\ref{sec:condition}, we observe that
    revealing the data label can curtail the candidate set of input data 
    and thus provide a significant gain in gradient inversion by using conditional generative model.
    Therefore, by making the label restoration challenging, 
    the gain from generative model may be decreasing. To be specific,
    we can consider letting node sample mini-batch to contain data having a certain number of 
    labels, less than the number of data but not too small. By doing this, the possible combinations of labels per data in a batch increases and thus the labels are hard to recovered.
    
\paragraph{Using large mini-batch.}
    There have been proposed several defense methods against gradient inversion attacks~\cite{fedavg, lin2018dgc, wei}, 
    which let the gradient contain only small amount of information per data.
    Once a gradient is computed from a large batch of data, the quality of the reconstructed data
    using gradient inversion attacks fall off significantly, including ours(GI-$z/w$) as shown in Figure~\ref{figure:compressed}.
    The performance (PSNR) of GI-$z/w$, GI-$x$, and GI-$x$+BN are degenerated as batch size grows.
    However, we found that the degree of degeneration in GI-$x$+BN is particularly greater than that in ours, 
    and from batch size 32, the advantage of utilizing BN statistics almost disappeared. 
    This is because the benefit of BN statistics is divided by the batch size 
    while the generative model helps each reconstruction in batch individually.
    
\paragraph{Adding gaussian noise to the gradient.}
    From the perspective of the differential privacy, 
    adding noise to the gradient can prevent gradient inversion attacks 
    from optimizing its objective.
    In our experiments, adding sufficiently large gaussian noise were able to prevent gradient inversion algorithms, including ours.
    \cite{zhu} also provided a similar observation. 
    Furthermore, we investigated that using large mini-batch with adding noise leverages the degree of degeneration.
    The result is shown in Figure~\ref{figure:compressed}.
    This also implies that one needs to add a larger noise when using a smaller batch size. 
    In addition, this justifies employing a mechanism of secure multi-party computation 
    with zero-sum antiparticles \cite{bonawitz2017secureaggregation} or zero-mean noises \cite{adfss} against our attack method. 
    However, such a mechanism may increase implementation complexity or learning instability. 

    % For instances, a gradient is computed from a large batch of data
    % and is reported to the server after compression or quantization.
    Such approaches can easily make the model training unstable.
    In general, we need to find a good balance in the trade-off between the stability of FL
    and the privacy leakage, while each defense mechanism has distinguishing pros and cons.

    % These methods usually entail performance the trade-off
    % between 
    % gradients can be made more difficult-to-recover at the cost of degraded model performance. 
    % Intuitively, such methods reduce the information about the input data contained in the gradients. This makes gradient inversion attack more difficult, while making model training less accurate at the same time.
    
%    \paragraph{Label ambiguity defense}
    % Most gradient inversion methods take the label reconstruction step as their first step, assuming that there are no duplicate labels in the underlying batch of samples.
    % This implies that one may carefully construct a batch so that this first step (i.e. recovering the labels) of the attack algorithms becomes nontrivial.
    % For instance, if one includes three images with the same label \emph{cat} in one batch, the current label reconstruction method can only detect that some samples with label \emph{cat} exist, but cannot infer how many such samples are there in the batch.
    % Note that this defense is more effective when the batch size is much larger than the number of classes.
    
% \tblue{
    
\section{Convergence speed comparison with GIML}
\label{sec:convergence_speed}
    
    \begin{figure}[ht]
      \centering
    %   \begin{subfigure}[b]{0.5\textwidth}
      
        \includegraphics[width=0.5\textwidth, trim={0cm, 0cm, 0cm, 0cm}]{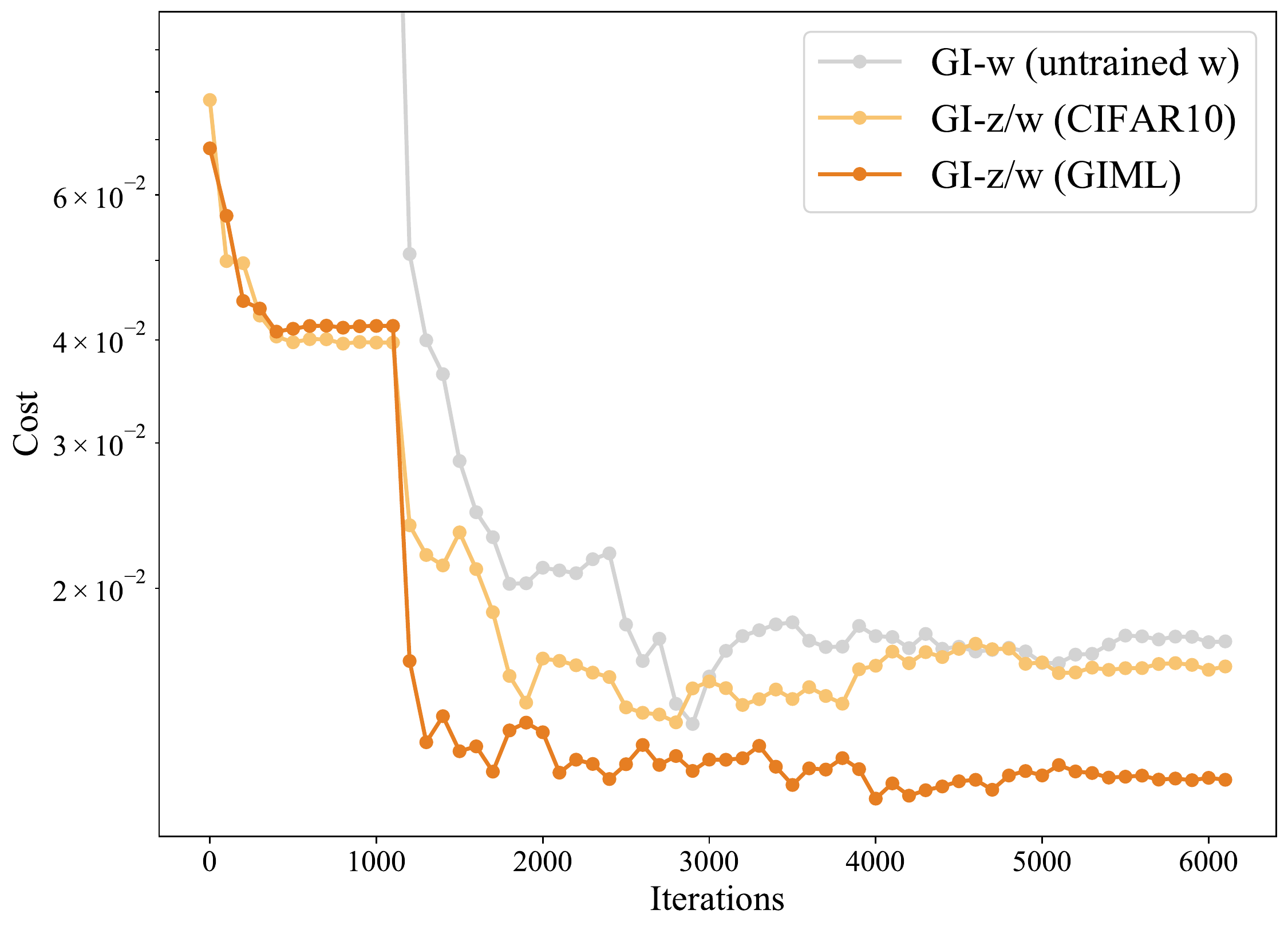}
        % \caption{}
        % \label{fig:trained_loss}
        
    %   \end{subfigure}
      \caption{
        Gain of using generative model trained with GIML.
          A typical loss curve of reconstruction process.
        %   GI-$w$ with untrained model, GI-$z/w$ with CIFAR10 pretrained model, and GI-$z/w$ with GIML are in comparison.
          Meta-learned model converges faster than other model's results.
          Note that GI-$w$ does not perform latent space search of 0 to 1000 iterations.
      }
      \label{figure:loss_trained}
    \end{figure}
    
    Meta learning algorithms such as MAML~\cite{finn17maml} and Reptile~\cite{nichol2018firstorder} 
    are often regarded as finding a good initialization for multiple tasks.
    In our case, each GIAS corresponds to a task.
    Thus, not only the performance of GIAS, the convergence speed of GIAS also increases.
    In Figure~\ref{figure:loss_trained}, we compare convergence speeds of GIAS with 
    a meta-learned generative model, an wrong generative model(GIML), and an untrained generative model.
    The result shows using GIML also boost up the convergence speed of GIAS.
    
% }

\section{Experiment settings}
\label{sec:exp-setting}

    Unless stated otherwise, 
    we consider an image classification task 
    on the validation set of ImageNet~\cite{imagenet} dataset 
    resized to $64 \times 64$
    using a randomly initialized ResNet18~\cite{he2016deep} as learning model. 
    The resizing is necessary for computational tractability. %, we rescaled the ImageNet images to 64x64. 
    Recalling the under-determined issue
    is the major challenge in gradient inversion, 
    deeper and wider $f_\theta$ makes the gradient inversion easier \cite{geiping}.
    Hence, the choice of ResNet18 as learning model is the most difficult setting within ResNet architectures since it contains the least information. 
    Considering a trained ResNet as learning model results in slight drop of quantitative performance and large variance.
    %than shallower and narrower ones.
    % \tblue{
    We use a StyleGAN2\cite{Karras2019stylegan2} model trained on ImageNet for GIAS,
    in which the latent space search over $z$ implies the search over the intermediate latent space, known as $\mathcal{W}$ in the original paper \cite{Karras_2019_CVPR},
    to improve the reconstruction performance, c.f., \cite{Abdal_2020_CVPR}.
    % }%, as done in \cite{Abdal_2020_CVPR}.
    %
    We use the batch size $B = 4$ as default,
    and negative cosine for the choice of gradient dissimilarity function $d(\cdot,\cdot)$,
    which apparently provides better inversion
    performance than $\ell_2$-distance in general \cite{geiping}. 
    For the optimization in GIAS, we use 
    Adam optimizer \cite{adam}
    which decays learning rate by a factor of $0.1$ at
    $3/8, 5/8, 7/8$ of total iterations.
    from initial learning rates $\eta_z = 3\times 10^{-2}$ for the latent space search and $\eta_w = 10^{-3}$ for the parameter space search. 
    Since our experiments are conducted with image data, we used total variation regularizer with
    weight $\lambda_\text{TV} = 10^{-4}$ for all experiments.
    For each inversion, we pick the best recovery
    among $4$ random instances based on the final loss.
        % cosine similarity $x$ learning rate $1e-1$,
        % l2 $x$ learning rate $3e-2$,
        % l2 $z$ learning rate $3e-2$,
        % $w$ learning rate $1e-3$,
        % total variation $1e-4$,
% If any experiment use images of other resolution or model, we mention those changes in each section.
    All experiments are performed on GPU servers equipped with NVIDIA RTX 3090 GPU and NVIDIA RTX 2080 Ti GPU.
    Numerical results including graphs and table are averaged over 10 samples except Figure~\ref{figure:images} and Figure~\ref{figure:loss_trained}.

    Note that our baseline implementation for \cite{nvidia} includes fidelity regularizer with BN$_\text{exact}$ and group lazy regularizer, not group registration regularizer. 
    Our baseline implementation might be imperfect, but it still demonstrates adding our method improves performance.

\section{License of assets}
    \paragraph{Dataset.}
        ImageNet data are distributed under licenses which allow free use for non-commercial research.
        FFHQ data are distributed under licenses which allow free use, redistribution, and adaptation for non-commercial purposes.
        ffhq-features-dataset provides annotations of FFHQ images. 
        Original authors of FFHQ images are indicated in the metadata, if required.
        We did not include, redistribute, or change the data itself, and cited above three works.
        Note that there are still some concerns related to whether the data owners of the original images or the people within the images provided informed consent for research use.
    
    \paragraph{Source code.}
        Some parts of our source code came from open-source codes of several previous researches.
        For more details, see README of our source code.

\end{document}